\PassOptionsToPackage{numbers, compress}{natbib}
\documentclass{article}
\usepackage[preprint]{neurips_2026}

\usepackage[utf8]{inputenc}
\usepackage[T1]{fontenc}
\usepackage{hyperref}
\usepackage{url}
\usepackage{booktabs}
\usepackage{amsfonts}
\usepackage{amsmath}
\usepackage{cancel}
\usepackage{amssymb}
\usepackage{amsthm}
\usepackage{nicefrac}
\usepackage{microtype}
\usepackage{xcolor}
\usepackage{caption}
\usepackage{graphicx}
\usepackage{multirow}
\usepackage{colortbl}
\usepackage{tikz}
\usetikzlibrary{calc,fit,positioning,arrows.meta,decorations.pathreplacing,matrix}

\newtheorem{theorem}{Theorem}
\newtheorem{proposition}[theorem]{Proposition}

\newtheorem{lemma}[theorem]{Lemma}

\newcommand{\bc}{\mathbf{c}}         
\newcommand{\bch}{\hat{\mathbf{c}}}  
\newcommand{\bs}{\mathbf{s}}         
\newcommand{\bG}{\mathbf{G}}         
\newcommand{\bpi}{\boldsymbol{\pi}}  

\newcommand{\bpsi}{\boldsymbol{\psi}}  
\newcommand{\R}{\mathbb{R}}
\newcommand{\Z}{\mathbb{Z}}
\newcommand{\E}{\mathbb{E}}
\newcommand{\Var}{\mathrm{Var}}

\newcommand{\argmin}{\operatorname{argmin}}

\title{Fitting Multilinear Polynomials for Logic Gate Networks}

\author{Youngsung Kim \\
Department of AI; and Department
of ECE,  \\
Inha University \\
\texttt{y.kim@inha.ac.kr;yskim.ee@gmail.com}}

\begin{document}

\maketitle

\begin{abstract}

We study learnable logic gate networks that stack layers of 2-input Boolean gates to build combinational circuits. Every 2-input gate has a unique multilinear polynomial with 4 coefficients, so the 16 Boolean gates form a codebook of prototypes in a 4-dimensional space, reducing training to a vector-quantization problem. The baseline method, Soft-Mix, learns a 16-dimensional softmax over gate identities, but the codebook has rank~4: 11 of 15 simplex directions carry nullspace gradient, and at uniform initialization the backward signal vanishes exactly. We prove that no affine product reparameterization fixes the resulting interaction-coefficient starvation under STE, and show that the covariance Jacobian of soft-VQ selection bypasses it by coupling the starved coefficient to the always-active constant channel. Working in the 4-dimensional polynomial space reduces each neuron from 16 to 4 parameters. On seven datasets, at least one 4-parameter method matches or exceeds Soft-Mix on every dataset; the CovJac advantage over STE grows monotonically with interaction demand across all seven datasets. At depth, Soft-Mix collapses ($-37.3$pp on CIFAR-10 at 12 layers) while CovJac holds ($-0.5$pp on CIFAR-10, stable on MNIST).

\end{abstract}


\section{Introduction}
\label{sec:introduction}

Learnable logic gate networks build combinational circuits by assigning a Boolean gate to every neuron~\citep{petersen2022deep,petersen2024convolutional}. These networks sit at the extreme efficiency end of the neural network spectrum: the deployed model is a combinational circuit requiring no floating-point arithmetic and no weight storage in the hidden layers~\citep{andronic2023polylut,bacellar2024dwn}, with inherent interpretability since each neuron computes a named logic function (AND, XOR, etc.)~\citep{jukna2012boolean}. The challenge is training: selecting the right gate per neuron is a discrete optimization problem with structural gradient pathologies. For 2-input gates (the setting studied here), training must choose one of 16 possible gates per neuron, a discrete selection problem typically relaxed via continuous surrogates~(\S\ref{sec:background}).

\paragraph{Observation: the 16 gates live in a 4-dimensional space.}
Every 2-input Boolean function has a unique multilinear polynomial with 4 coefficients (\S\ref{sec:mle})~\citep{hammer1968boolean,owen1972multilinear}, so the 16 gates are 16 integer points in a 4-dimensional space. The softmax operates in 15 dimensions, but only 4 of those directions actually change the neuron's output; the remaining 11 carry gradient that does no work (\S\ref{sec:method}). This suggests training directly in the 4-dimensional coefficient space instead, learning $\mathbf{c}\in\mathbb{R}^4$ and snapping to the nearest valid gate.

\paragraph{Two gradient pathologies in this space.}
Working in 4 dimensions eliminates the redundancy but exposes two problems. \textbf{(i)}~The soft mixture used by the existing method suffers from exact gradient cancellation: at the standard uniform initialization, the backward signal through each neuron is provably zero for every input (Proposition~\ref{prop:softmix-cancel}). This worsens with depth. \textbf{(ii)}~Under the straight-through estimator (STE), the interaction coefficient $c_{ab}$, which distinguishes AND from OR and XOR from XNOR, receives gradient on only 25\% of training samples. We prove that no affine reparameterization of the polynomial basis can fix this (Proposition~\ref{prop:basis-futility}).

\paragraph{Solution: covariance Jacobian coupling.}
Our second method, Multilinear-CovJac, uses the Jacobian of a soft vector-quantization step to couple $c_{ab}$ to the always-active constant channel, routing gradient to it on every sample (Proposition~\ref{prop:covjac}). Both resulting methods, Multilinear-STE and Multilinear-CovJac, use 4 parameters per neuron (Figure~\ref{fig:intro}).


\definecolor{cSM}{HTML}{C62828}    
\definecolor{cSMbg}{HTML}{FFEBEE}
\definecolor{cML}{HTML}{0072B2}    
\definecolor{cMLbg}{HTML}{DCEEF7}
\definecolor{cFx}{HTML}{212121}    
\definecolor{cFBg}{HTML}{E0E0E0}

\begin{figure}[t]
\centering
\resizebox{0.95\linewidth}{!}{%
\begin{tikzpicture}[x=1cm,y=1cm,
  arr/.style={-{Latex[length=2.5pt]},line width=0.5pt,cFx!70},
  sub/.style={font=\tiny,text=cFx!80,anchor=north,yshift=-2pt,align=center},
]

\def\Ya{0.65}

\node[font=\footnotesize\bfseries,text=cSM,anchor=east] at(-0.1,\Ya)
  {(a) Soft-Mix};

\node[draw=cSM,fill=cSMbg,line width=0.7pt,rounded corners=2pt,
  minimum height=0.45cm,inner sep=4pt,
  font=\footnotesize\bfseries,text=cSM](pi) at(1.3,\Ya)
  {$\boldsymbol{\pi}$};
\node[sub] at(pi.south){$\in\!\Delta^{15}$\\[-1pt]16 params};

\node[draw=cFx!70,fill=cFBg,line width=0.6pt,rounded corners=2pt,
  minimum height=0.45cm,inner sep=4pt,
  font=\footnotesize\bfseries,text=cFx](G) at(2.7,\Ya)
  {$\mathbf{G}$};
\draw[arr](pi.east)--(G.west);
\node[sub] at(G.south){$\mathbb{Z}^{16\times4}$\\[-1pt]rank 4};

\node[draw=cFx!70,fill=cFBg,line width=0.6pt,rounded corners=2pt,
  minimum height=0.45cm,inner sep=4pt,
  font=\footnotesize,text=cFx](ca) at(4.2,\Ya)
  {$\mathbf{c}$};
\draw[arr](G.east)--(ca.west)
  node[midway,above,font=\tiny,text=cFx!70]{$\boldsymbol{\pi}^\top\!\mathbf{G}$};

\node[draw=cFx!70,fill=white,line width=0.6pt,rounded corners=2pt,
  minimum height=0.45cm,inner sep=4pt,
  font=\footnotesize,text=cFx](za) at(5.65,\Ya)
  {$z$};
\draw[arr](ca.east)--(za.west)
  node[midway,above,font=\tiny,text=cFx!70]{$\mathbf{c}^\top\!\boldsymbol{\psi}$};

\draw[arr](za.east)--++(0.5,0);
\node[font=\scriptsize,text=cSM,anchor=west](ga)
  at([xshift=14pt]za.east){$\partial z/\partial a = 0$ {\tiny at uniform $\boldsymbol{\pi}$}};

\draw[decorate,decoration={brace,amplitude=2.5pt,raise=2pt},cSM!50]
  (pi.north west)--(G.north east)
  node[midway,above=5pt,font=\tiny,text=cSM!70]
  {11/15 simplex directions $\to$ nullspace};

\def\Yb{-0.75}

\node[font=\footnotesize\bfseries,text=cML,anchor=east] at(-0.1,\Yb)
  {(b) Multilinear (ours)};

\node[draw=cML,fill=cMLbg,line width=0.7pt,rounded corners=2pt,
  minimum height=0.45cm,inner sep=4pt,
  font=\footnotesize\bfseries,text=cML](cb) at(2.7,\Yb)
  {$\mathbf{c}$};
\node[sub] at(cb.south){$\in\!\mathbb{R}^4$\\[-1pt]4 params};

\node[draw=cFx!70,fill=cFBg,line width=0.6pt,rounded corners=2pt,
  minimum height=0.45cm,inner sep=4pt,
  font=\footnotesize,text=cFx](chat) at(4.2,\Yb)
  {$\hat{\mathbf{c}}$};
\draw[arr](cb.east)--(chat.west)
  node[midway,above,font=\tiny,text=cFx!70]{snap to $\bG$};
\node[sub] at(chat.south){$\in\bG$\\[-1pt]quantized};

\node[draw=cFx!70,fill=white,line width=0.6pt,rounded corners=2pt,
  minimum height=0.45cm,inner sep=4pt,
  font=\footnotesize,text=cFx](zb) at(5.65,\Yb)
  {$z$};
\draw[arr](chat.east)--(zb.west)
  node[midway,above,font=\tiny,text=cFx!70]{$\hat{\mathbf{c}}^\top\!\boldsymbol{\psi}$};

\draw[arr](zb.east)--++(0.5,0);
\node[font=\scriptsize,text=cML,anchor=west](gb)
  at([xshift=14pt]zb.east){$\partial z/\partial a = c_a\!+\!c_{ab}b$};
\node[font=\tiny,text=cML!80,anchor=north east]
  at(gb.south east){$\neq 0$ after snap};

\pgfmathsetmacro{\Ymid}{(\Ya+\Yb)/2}
\node[font=\scriptsize,text=cFx!80](psi) at(5.65,\Ymid)
  {$\boldsymbol{\psi}\!=\![1,\,a,\,b,\,ab]$};
\draw[cFx!40,line width=0.3pt,-{Latex[length=1.5pt]}](psi.north)--(za.south);
\draw[cFx!40,line width=0.3pt,-{Latex[length=1.5pt]}](psi.south)--(zb.north);

\end{tikzpicture}
}%

\caption{Both methods differ in how the coefficient vector
$\mathbf{c}$ is obtained.
Colored blocks are \textbf{learnable parameters}
(\textcolor{cSM}{red}: Soft-Mix, \textcolor{cML}{blue}: ours);
gray blocks are fixed or derived.
\textbf{(a)}~Soft-Mix learns $\boldsymbol{\pi}\!\in\!\Delta^{15}$
(16 params) and projects through the rank-4 codebook $\mathbf{G}$:
11 of 15 simplex directions carry nullspace gradient, and at uniform
$\boldsymbol{\pi}$ the input derivative vanishes exactly
(Prop.~\ref{prop:softmix-cancel}).
\textbf{(b)}~Our method learns
$\mathbf{c}\!\in\!\mathbb{R}^4$ (4 params) and quantizes to the
nearest codebook entry
$\hat{\mathbf{c}}\!\in\!\mathbf{G}$; the output is
$z = \hat{\mathbf{c}}^\top\!\boldsymbol{\psi}$, and after snapping
$|c_a + c_{ab}b| \geq 1$ always
(Prop.~\ref{prop:ste-noncancel}).}
\label{fig:intro}
\end{figure}
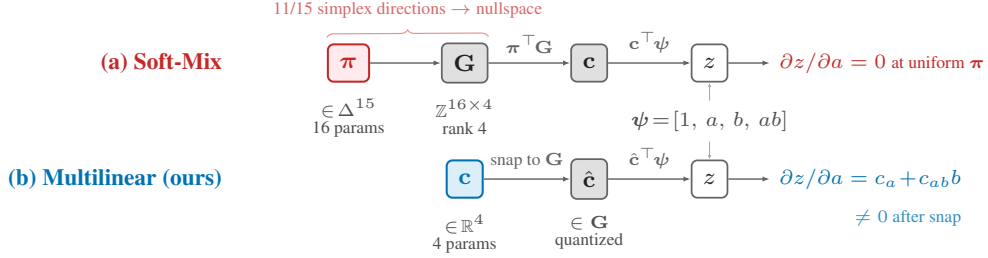

\paragraph{Contributions.}
\textbf{(1)}~We identify the rank-4 structure of the Boolean gate codebook and show that 11 of 15 softmax gradient directions are wasted; at uniform initialization the gradient vanishes exactly. \textbf{(2)}~We prove that the interaction coefficient $c_{ab}$ is inherently starved under STE, and that no affine product basis can simultaneously fix coverage, coherence, and bias. \textbf{(3)}~We propose two 4-parameter gate layers: Multilinear-STE (hard snap + polynomial STE) and Multilinear-CovJac (soft-VQ with input-independent Jacobian coupling to $c_{ab}$). \textbf{(4)}~We validate on seven datasets that at least one 4-parameter method matches or exceeds 16-parameter Soft-Mix on every dataset; the CovJac advantage grows monotonically with task interaction demand; at $L{=}12$, Soft-Mix collapses ($-37.3$pp on CIFAR-10) while CovJac holds ($-0.5$pp).


\section{Background}
\label{sec:background}

\subsection{Training logic gate networks}

Selecting one of 16 known operations per neuron is a discrete optimization problem analogous to operator selection in neural architecture search~\citep{liu2019darts}, but over a small, fully enumerable set. The standard approach is continuous relaxation: replace the discrete gate choice with a differentiable surrogate so the network can be trained end-to-end with backpropagation, then discretize at deployment. Differentiable logic gate networks (DLGNs)~\citep{petersen2022deep}, the architecture we study, implement this with $L$ hidden layers of width $k$. Each neuron in layer $l$ receives two inputs from layer $l{-}1$ via fixed random wiring; the first layer draws from the input features. The final layer's $k$ outputs are partitioned into $C$ equal groups (one per class), summed within each group, and passed through softmax to produce class probabilities (GroupSum readout).

\paragraph{Soft mixture over all 16 gates.}
In the approach of \citet{petersen2022deep}, each neuron learns a distribution $\bpi\in\Delta^{15}$ over the 16 gates; the training output is the weighted mixture $z = \sum_{j=0}^{15} \pi_j\, g_j(a,b)$, and the deployed gate is $g^* = \arg\max_g \pi_g$. Because training uses a soft mixture of all 16 gates while deployment uses a single selected gate, there is a potential train-to-deploy mismatch~\citep{kim2026alignforward}. \citet{kim2024narrowing} further analyzed the generalization gap and introduced logical skip connections to narrow it.

\paragraph{Single-gate selection via Gumbel-Softmax.}
Keeps 16 logits per neuron but samples a one-hot gate via Gumbel-Softmax~\citep{jang2017categorical,maddison2017concrete} during training (hard argmax at deployment), with straight-through gradients~\citep{bengio2013estimating}. Only one gate is active per forward pass, avoiding the backward cancellation (\S\ref{sec:softmix-cancel}). \citet{kim2023deep} introduced Gumbel-Softmax with STE to logic gate networks to reduce the discretization gap. \citet{yousefi2025mindthegap} further analyze the discretization gap in DLGNs and propose methods to reduce it.

\paragraph{Input-wise Parametrization (IWP).}
\citet{ruttgers2026narrowing} reduce to 4 parameters per neuron in the corner basis $\phi_{ij}(a,b)=a^i(1{-}a)^{1-i}b^j(1{-}b)^{1-j}$ and identifies that complementation symmetry ($g\mapsto 1{-}g$) causes pairwise gradient cancellation at symmetric initialization, stalling optimization at depth.

\subsection{The multilinear extension}
\label{sec:mle}

Every 2-input Boolean function has a unique multilinear polynomial~\citep{hammer1968boolean,owen1972multilinear,odonnell2014analysis,jukna2012boolean}. In the canonical monomial basis this is
\begin{equation}
g(a,b) = c_0 + c_a\,a + c_b\,b + c_{ab}\,ab, \qquad \bc=[c_0,c_a,c_b,c_{ab}]^\top\in\R^4,
\label{eq:mle}
\end{equation}
where $c_0$ is the constant output, $c_a$ and $c_b$ are linear effects, and $c_{ab}$ is the interaction term, zero for separable gates, nonzero for interactive gates. Since the inputs are Boolean, all 16 gates have integer coefficients, e.g.\ $\text{AND}{=}[0,0,0,1]$, $\text{OR}{=}[0,1,1,{-}1]$, $\text{XOR}{=}[0,1,1,{-}2]$, $\text{A}{=}[0,1,0,0]$, $\text{FALSE}{=}[0,0,0,0]$ (full table in Appendix~\ref{app:identity1}).

The same polynomial can be written in the corner basis $\phi_{ij}(a,b){=}a^i(1{-}a)^{1-i}b^j(1{-}b)^{1-j}$, where each basis function isolates one input combination; this is the parameterization used by IWP~\citep{ruttgers2026narrowing}. The two bases are related by an invertible linear map (Appendix~\ref{app:cost}) and span the same function space; the critical difference is the gradient structure under discrete training constraints, analyzed in \S\ref{sec:theory}. Since the canonical basis has the sparsest monomials under STE (\S\ref{sec:theory}), we train directly in this 4-dimensional coefficient space. Connections to fuzzy logic, probabilistic circuits, and degree as logical complexity are in Appendix~\ref{app:background}.


\section{Method}
\label{sec:method}

We collect the 16 integer coefficient vectors into a codebook $\bG\in\Z^{16\times 4}$ of rank~4 ($\bG_j$: row $j$, a 4-vector; $G_{jk}$: scalar entry), and observe that the Soft-Mix output can be written as $z = \bc_{\mathrm{eff}}^\top\bpsi(a,b)$ with $\bc_{\mathrm{eff}}=\bpi^\top\bG$ and $\bpsi=[1,a,b,ab]^\top$. Since $\bG$ has rank~4, the map $\bpi\mapsto\bpi^\top\bG$ projects the 15-dimensional simplex onto a 4-dimensional space: \textbf{only 4 of 15 simplex directions change the neuron's output; the remaining 11 carry gradient that does no work.}

We therefore train directly in the 4-dimensional coefficient space. Each neuron stores $\bc=[c_0,c_a,c_b,c_{ab}]^\top\in\R^4$ (4 learnable parameters). Because the deployed circuit requires a valid Boolean gate, $\bc$ must be quantized to the nearest codebook entry:
\begin{equation}
\bch = \bG_{j^*},\qquad j^* = \argmin_j \|\bc-\bG_j\|_2^2.
\label{eq:snap}
\end{equation}
By Eq.~\ref{eq:mle}, $z = \bch^\top\bpsi(a,b)$ agrees with the Boolean truth table on $\{0,1\}^2$. The two methods below share this quantization and differ only in \emph{when} it is applied and \emph{how} gradients flow through it.

\subsection{Multilinear--Straight-Through Estimator (STE)}
\label{sec:multilinear-ste}

\paragraph{Forward.}
Each neuron snaps $\bc$ to the nearest gate $\bch$ via Eq.~\ref{eq:snap} and evaluates $z = \bch^\top\bpsi(a,b) = \hat c_0 + \hat c_a a + \hat c_b b + \hat c_{ab}\,ab$. The same quantized gate is used in both training and deployment, so there is \textbf{no train-to-deploy mismatch} (Appendix~\ref{app:identity1}).

\paragraph{Backward.}
The quantization $\bc\mapsto\bch$ (Eq.~\ref{eq:snap}) maps every $\bc$ in a region of $\R^4$ to the same codebook vector, so small perturbations of $\bc$ leave $\bch$ unchanged and $\partial\bch/\partial\bc = \mathbf{0}$ almost everywhere. The straight-through estimator (STE; \citet{bengio2013estimating}) bypasses this by approximating $\partial\bch/\partial\bc \approx \mathbf{I}$ in the backward pass. Let $\mathcal{L}$ denote the network loss and $\delta = \partial\mathcal{L}/\partial z$ the upstream error signal backpropagated to this neuron's output $z = \bch^\top\bpsi$. Then by the chain rule:
\begin{equation}
\frac{\partial\mathcal{L}}{\partial\bc} \approx \frac{\partial\mathcal{L}}{\partial\bch} = \frac{\partial\mathcal{L}}{\partial z} \cdot\frac{\partial z}{\partial\bch} = \delta\cdot\bpsi(a,b) = \delta\cdot[1,\,a,\,b,\,ab]^\top,
\label{eq:ste}
\end{equation} The gradient for each
coefficient $c_k$ is scaled by the corresponding basis monomial $[1, a, b, ab]_k$, which is nonzero only when all its input arguments equal~1. This creates a \textbf{hierarchical} coverage pattern: under i.i.d.\ Bernoulli$(1/2)$ inputs (idealized; confirmed empirically in \S\ref{sec:experiments}), $c_0$ receives gradient 100\% of the time (its monomial is~1), $c_a$ and $c_b$ at 50\% (active when $a{=}1$ or $b{=}1$), and $c_{ab}$ at only 25\% (active only when $a{=}b{=}1$). The expected active-component count is $1{+}\tfrac{1}{2}{+}\tfrac{1}{2}{+}\tfrac{1}{4}=2.25$, versus $4{\times}\tfrac{1}{4}=1.0$ for the corner basis, a $2.25\times$ ratio (Lemma~\ref{lem:coverage}; full table in Appendix~\ref{app:coverage}). The empirical consequence (a 17pp accuracy gap) is in \S\ref{sec:experiments}.

However, this hierarchy has a cost. \textbf{The interaction coefficient $c_{ab}$, the interaction term (second-order Sobol index~\citep{sobol1993sensitivity,owen2013variance}) that distinguishes AND from OR and XOR from XNOR, receives gradient on only 1 in 4 samples.} This is not an implementation artifact: it is intrinsic to the multilinear basis under STE, and no affine reparameterization of the basis can fix it without introducing bias or losing coherence (Proposition~\ref{prop:basis-futility}, \S\ref{sec:basis-futility}). The $c_{ab}$ starvation motivates the second variant below.

\subsection{Multilinear--Covariance Jacobian (CovJac)}
\label{sec:multilinear-covjac}

Under STE, the interaction coefficient $c_{ab}$ learns from only 25\% of training samples while $c_0$ learns from all of them. We seek a gradient mechanism that couples $c_{ab}$ to the always-active dimensions, routing gradient to it on every sample.

\paragraph{Idea: soft vector quantization.}
Instead of hard-quantizing $\bc$ to the nearest codebook entry (as in STE), we relax the quantization during training by computing a proximity-weighted average over all codebook entries~\citep{oord2017neural,agustsson2017softhard}:
\begin{equation}
\omega_j = \mathrm{softmax}\!\bigl(-\|\bc-\bG_j\|^2/\tau\bigr)_j, \qquad \bc_{\text{soft}} = \sum_j \omega_j\,\bG_j.
\label{eq:softvq}
\end{equation}
Here $\omega_j$ are proximity weights (not the Soft-Mix gate distribution $\bpi$): codebook entries closer to $\bc$ receive higher weight. During training the neuron evaluates $z = \bc_{\text{soft}}^\top\bpsi(a,b)$; at deployment it uses the hard-quantized $\bch$ as before. We fix $\tau{=}1$ throughout. The codebook is the fixed set of 16 Boolean gates, not learned.

\paragraph{Why this helps: the covariance Jacobian.}
The key property of soft-VQ is that gradients flow through the Jacobian $\partial\bc_{\text{soft}}/\partial\bc$. Since $\omega_j \propto \exp(-\|\bc-\bG_j\|^2/\tau)$, the derivative of each weight is $\partial\omega_j/\partial\bc = (2/\tau)\,\omega_j(\bG_j - \bc_{\text{soft}})$; substituting into $\partial\bc_{\text{soft}}/\partial\bc = \sum_j (\partial\omega_j/\partial\bc)\,\bG_j^\top$ gives
\begin{equation}
\mathbf{J} = \frac{\partial\bc_{\text{soft}}}{\partial\bc} = \frac{2}{\tau}\,\mathrm{Cov}_{\boldsymbol{\omega}}(\bG),
\label{eq:covjac}
\end{equation}
the $\omega$-weighted covariance of the codebook, $\mathrm{Cov}_\omega(\bG)_{ik} = \sum_j\omega_j(G_{ji}-(\bc_{\text{soft}})_i)(G_{jk}-(\bc_{\text{soft}})_k)$ (full proof in Appendix~\ref{app:ste-covjac}; see also~\citet{huh2023straightening} for STE failure modes motivating soft alternatives). This matrix is symmetric PSD, strictly positive on the diagonal, and and, crucially, \textbf{independent of the input} $(a,b)$ (Proposition~\ref{prop:covjac}, \S\ref{sec:covjac}). Because $\mathbf{J}$ is a covariance matrix over the codebook, its off-diagonal entries couple different coefficient dimensions. Writing $J_{ik}$ for the $(i,k)$ entry of $\mathbf{J}$ (0-indexed: $i{=}0$ for $c_0$, $i{=}3$ for $c_{ab}$) and $\delta=\partial\mathcal{L}/\partial z$ for the upstream scalar gradient, the gradient reaching $c_{ab}$ is:
\begin{equation}
\frac{\partial\mathcal{L}}{\partial c_{ab}} = \delta\cdot\bigl(\underbrace{1\cdot J_{03}}_{\text{always active}} + a\cdot J_{13} + b\cdot J_{23} + ab\cdot J_{33}\bigr).
\label{eq:cab-gradient}
\end{equation}
The leading term $J_{03}$ multiplies the constant monomial~1, which is active on every sample. Since $J_{03} = (2/\tau)\,\mathrm{Cov}_{\omega}(G_{\cdot 0},\, G_{\cdot 3}) \neq 0$ generically, \textbf{$c_{ab}$ receives a non-zero gradient on every training sample}, not just the 25\% where $a{=}b{=}1$.

\paragraph{Comparison of the two mechanisms.}
Both methods reduce each neuron from 16 to 4 parameters. The two gradient mechanisms are complementary: STE is \emph{support-weighted} (gradient fires when the relevant input is active) and quantizes in both phases (no train-deploy mismatch); CovJac is \emph{codebook-coupling-weighted} (gradient fires via $J_{03}$ regardless of input) and uses soft-VQ during training, quantizing only at deployment. The price of CovJac is a small train-deploy gap (${\leq}\,0.2\%$ on all main datasets; ${\leq}\,1\%$ on MONK's-2 with $k{=}136$). The $4\times$ parameter reduction yields a $4\times$ reduction in Adam optimizer state. A full method comparison is in Appendix~\ref{app:cost}.


\section{Theoretical Analysis}
\label{sec:theory}

Section~\ref{sec:method} claimed that Soft-Mix gradient cancels at uniform $\bpi$, that Multilinear-STE avoids this cancellation, and that CovJac couples $c_{ab}$ to an always-active channel. We now prove these claims. We also address whether a different polynomial basis could fix the $c_{ab}$ starvation. We show that the answer is no: no affine product basis can simultaneously achieve full coverage, coherent direction, and unbiased updates for $c_{ab}$ under STE. This is why CovJac (a different gradient mechanism, not a different basis) is needed.

\subsection{Soft-Mix backward cancellation}
\label{sec:softmix-cancel}

We denote $c_{a,j}$, $c_{b,j}$, $c_{ab,j}$ for the coefficients of gate~$j$ (i.e.\ the entries of $\bG_j$ from \S\ref{sec:method}).

\begin{lemma}[Zero-sum gate derivative symmetry; elementary]
\label{lem:zerosum}
For the 16 Boolean coefficient vectors,
\begin{equation}
\sum_{j=0}^{15} c_{a,j} = \sum_{j=0}^{15} c_{b,j} = \sum_{j=0}^{15} c_{ab,j} = 0.
\label{eq:zerosum}
\end{equation}
\end{lemma}

\emph{Proof sketch.} Boolean functions are closed under pointwise complementation $g\mapsto\bar g=1-g$, which is an involution with no fixed points; the set $\mathcal{B}$ of 16 Boolean functions therefore splits into 8 complementary pairs. By Eq.~\ref{eq:mle}, $\bar g = 1-g = (1{-}c_0) - c_a a - c_b b - c_{ab} ab$, so $\bar c_a = -c_a$, $\bar c_b = -c_b$, $\bar c_{ab} = -c_{ab}$. Each complementary pair sums to zero in the non-constant coefficients. Summing pair-by-pair gives the claim. \emph{(Attribution: the complementation pairing is the same observation used by~\citet[Eq.~7]{ruttgers2026narrowing} in the corner basis. The zero-sum follows from standard Boolean function symmetry~\citep[Ch.~1]{crama2011boolean,odonnell2014analysis}.)}

\begin{proposition}[Soft-Mix backward vanishes at uniform $\bpi$]
\label{prop:softmix-cancel}
\citet[Eq.~7]{ruttgers2026narrowing} showed that this complementation symmetry causes approximate gradient cancellation at symmetric initialization. We sharpen this to an exact identity that holds for all inputs, not only at initialization: for $z=\sum_{j=0}^{15}\pi_j g_j(a,b)$ and $\pi_j=1/16$ for all $j$,
\begin{equation}
\frac{\partial z}{\partial a}(a,b) = \frac{\partial z}{\partial b}(a,b) = 0 \quad\text{for every }(a,b)\in[0,1]^2.
\label{eq:softmix-cancel}
\end{equation}
\end{proposition}

\emph{Proof sketch.} Substituting the multilinear form and differentiating, $\partial z/\partial a = \sum_j\pi_j(c_{a,j}+c_{ab,j}b)$. At $\pi_j=1/16$, this equals $(1/16)[\sum_j c_{a,j} + b\sum_j c_{ab,j}]=0$ by Lemma~\ref{lem:zerosum}. The same holds for $\partial z/\partial b$.

\paragraph{Dynamical interpretation (Observation~1; informal, not a theorem).}
A near-uniform bound (proved in Appendix~\ref{app:cancellation}) gives $|\partial z/\partial a|\leq \|\bpi-\mathbf{u}\|_1$ where $\mathbf{u}=(1/16,\ldots,1/16)$ is the uniform distribution, so if $\|\bpi^{(l)}-\mathbf{u}\|_1\leq\epsilon$ at layer~$l$, the per-layer signal attenuation is at most~$\epsilon$.  At initialization with logit std $\sigma{=}1.0$, the expected $\|\bpi-\mathbf{u}\|_1\approx 0.72$, giving per-layer signal~${\leq}\,0.72$; over $L$ layers the product is $O(0.72^L)$, which at $L{=}12$ is~${\approx}\,0.019$. \emph{What is rigorous}: the per-neuron bound $\|\bpi-\mathbf{u}\|_1$. \emph{What is not}: that neurons remain near-uniform throughout training rather than committing quickly; this is confirmed empirically (\S\ref{sec:diagnostics}, Figure~\ref{fig:cancellation}).

\paragraph{Single-gate-forward corollary.}
Proposition~\ref{prop:softmix-cancel}'s cancellation requires summing over all 16 gates: only then do the Lemma~\ref{lem:zerosum} zero-sums cancel inside $\partial z/\partial a$. Any method committing to a \emph{single} gate in the forward pass (whether by hard quantization or stochastic sampling~\citep{jang2017categorical,maddison2017concrete}) avoids the cancellation by construction, regardless of parameter count. This prediction is confirmed by the depth experiments (\S\ref{sec:depth}, Figure~\ref{fig:depth-scaling}).

\subsection{Multilinear-STE avoids the cancellation}
\label{sec:multilinear-ste-theory}

\begin{proposition}[Multilinear-STE backward non-cancellation]
\label{prop:ste-noncancel}
In Multilinear-STE with committed gate $\bch=(c_0,c_a,c_b,c_{ab})$, the STE backward derivative is $\partial z/\partial a = c_a + c_{ab}b$, and for every Boolean gate depending on input $a$, $\max_{b\in\{0,1\}}|c_a+c_{ab}b|\geq 1$.
\end{proposition}

\emph{Proof sketch.} The STE rule treats $\bch$ as a parameter; differentiating the polynomial gives $c_a + c_{ab}b$. The gate depends on $a$ iff $(c_a,c_{ab})\neq(0,0)$; case analysis on $c_a\neq 0$ and $c_a=0,c_{ab}\neq 0$ evaluated at $b=0$ or $b=1$ respectively each yields a value of magnitude $\geq 1$ since the integer coefficients are in $\{-2,\ldots,2\}$. Full proof and a 12-row enumeration of all $a$-dependent gates are in Appendix~\ref{app:ste-covjac}.

\subsection{CovJac selection Jacobian}
\label{sec:covjac}

Proposition~\ref{prop:ste-noncancel} shows that STE avoids cancellation but still starves $c_{ab}$ (25\% coverage). We now show that the CovJac Jacobian has three properties that together guarantee $c_{ab}$ receives gradient on \emph{every} sample.

\begin{proposition}[CovJac input-independence, PSD, positive diagonal]
\label{prop:covjac}
For $\bc\mapsto\bc_{\text{soft}}(\bc) = \sum_j\omega_j\bG_j$, the Jacobian is
\begin{equation}
\mathbf{J}(\bc) = \frac{2}{\tau}\,\mathrm{Cov}_{\boldsymbol{\omega}}(\bG),
\label{eq:covjac-theory}
\end{equation}
i.e.\ (i) input-independent (depends only on $\bc$, not $(a,b)$); (ii) symmetric PSD; (iii) strictly positive diagonal for every coordinate throughout training at $\tau{>}0$ on the 16-Boolean codebook.
\end{proposition}

\emph{Proof sketch.} The softmax Jacobian is a standard covariance matrix (see e.g.\ \citet[Eq.~4.106]{bishop2006pattern} for the softmax case). Here the chain rule gives $\partial\omega_j/\partial c_m = (2/\tau)\omega_j(G_{jm}-(\bc_{\text{soft}})_m)$. Since $(\bc_{\text{soft}})_i = \sum_j \omega_j G_{ji}$, differentiating by $c_m$ and substituting yields $J_{im} = (2/\tau)\bigl[\sum_j\omega_j G_{ji}G_{jm} - (\sum_j\omega_j G_{ji})(\sum_j\omega_j G_{jm})\bigr] = (2/\tau)\,\mathrm{Cov}_\omega(G_{\cdot i},G_{\cdot m})$, completing (i). For any $\mathbf{v}$, $\mathbf{v}^\top\mathbf{J}\mathbf{v} = (2/\tau)\Var_\omega(\sum_i v_i G_{\cdot i})\geq 0$ gives PSD. The diagonal $J_{mm}=(2/\tau)\Var_\omega(G_{\cdot m})>0$ since every coordinate takes at least two distinct values on the 16-Boolean codebook and softmax at finite $\tau$ has full support. Full proof in Appendix~\ref{app:ste-covjac}.

\noindent The cross-coefficient coupling consequence ($c_{ab}$ receives
gradient on every sample via $J_{03}$, Eq.~\ref{eq:cab-gradient}) follows directly from Proposition~\ref{prop:covjac}'s positive diagonal and input-independence.

\subsection{Basis futility}
\label{sec:basis-futility}

A natural question is whether a different polynomial basis could fix the $c_{ab}$ starvation without changing the gradient mechanism. The following proposition shows that the answer is no: any affine reparameterization of the basis trades one desirable property for another.

\begin{proposition}[No affine product basis simultaneously fixes STE]
\label{prop:basis-futility}
Let $\phi(x)=\alpha+\beta x$ with $\beta\neq 0$ and $\tilde{\bpsi}(a,b)=(1,\phi(a),\phi(b),\phi(a)\phi(b))^\top$. Define: \emph{coverage} $\rho = \mathbb{P}(\tilde\psi_3(a,b)\neq 0)$ (fraction of samples where the interaction component receives gradient), \emph{coherence} $\kappa = |\E[\tilde\psi_3]|/\E[|\tilde\psi_3|]$ (alignment of the gradient direction), and \emph{STE bias} $B = \inf_{D} \E[\|\tilde{\bpsi} - D\bpsi_{\mathrm{can}}\|^2]$ where $D$ ranges over positive diagonal matrices (distance to the canonical basis modulo positive rescaling). Full definitions in Appendix~\ref{app:basis-futility}. Then no affine product basis simultaneously attains $\rho=1$, $\kappa=1$, and $B=0$. In particular: the canonical basis $\phi(x){=}\beta x$ ($\beta{>}0$) is the unique $B{=}0$ case, with $\rho=1/4,\kappa=1$; all other affine bases give $\rho=1$ but either $\kappa<1$ or $B>0$ (or both).
\end{proposition}

\emph{Proof sketch.} Parameterize $(p,q)=(\phi(0),\phi(1))$. The canonical basis ($p{=}0$, $\phi(x){=}\beta x$) gives $(\rho,\kappa,B)=(1/4,1,0)$: low coverage but no bias. All other choices have $\rho{=}1$ but introduce either coherence loss ($\kappa{<}1$, e.g.\ Walsh basis with $p{=}{-}q$ gives $\kappa{=}0$) or STE bias ($B{>}0$, e.g.\ smoothed basis with $p,q{>}0$). Full case-split in Appendix~\ref{app:basis-futility}; empirical confirmation in Appendix~\ref{app:ste-failures}.


\section{Experiments}
\label{sec:experiments}

We validate the theoretical predictions on seven binary-input datasets: Adult, Splice~\citep{dua2019uci}, MNIST~\citep{lecun1998mnist}, SVHN~\citep{netzer2011svhn}, CIFAR-10/100~\citep{krizhevsky2009cifar}, and MONK's-2~\citep{thrun1991monks}. All experiments use the DLGN architecture~\citep{petersen2022deep} with $L{=}6$ layers, GroupSum readout, Adam~\citep{kingma2015adam} (lr$=0.01$), batch 512, and 3 seeds. We report \textbf{last-10 accuracy} (mean of final 10 evaluation checkpoints). We compare \emph{Soft-Mix} (16p), \emph{Gumbel-ST} (16p), \emph{Multilinear-STE} (ours, 4p), and \emph{Multilinear-CovJac} (ours, 4p); CovJac uses $\tau{=}1$ except Adult ($\tau{=}0.1$; Appendix~\ref{app:tau}). Full setup details in Appendix~\ref{app:hyperparams}; additional results including depth tables, width scaling, architecture baselines, and temperature sensitivity in Appendix~\ref{app:experiments-details}.

\subsection{Cross-dataset comparison}
\label{sec:cross-dataset}

Table~\ref{tab:mnist-main} shows that CovJac matches or exceeds all baselines on every dataset using $4\times$ fewer parameters. The CovJac advantage over Soft-Mix is small on easy tasks (+0.01pp on MNIST) but grows on harder ones (+0.84pp on CIFAR-10, +0.45pp on CIFAR-100 with the tightest variance). On Adult, the default $\tau{=}1.0$ causes CovJac's $c_{ab}$ coupling to add overhead on this low-interaction task; lowering $\tau$ to $0.1$ reduces the coupling and recovers M-STE-level accuracy (84.93\%).

The most informative comparison is CovJac vs STE, since both use 4 parameters and differ only in the gradient mechanism. This gap widens monotonically with task interaction demand: $+0.02$pp (Adult) $\to$ $+0.13$pp (Splice) $\to$ $+0.21$pp (MNIST) $\to$ $+2.22$pp (SVHN) $\to$ $+2.95$pp (CIFAR-10) $\to$ $+4.35$pp (CIFAR-100) $\to$ $+7.53$pp (MONK's-2; Table~\ref{tab:h7-cross-dataset} and Figure~\ref{fig:interaction-scaling} in Appendix~\ref{app:h7}). On the most interaction-dense tasks (CIFAR-10/100), STE underperforms even Soft-Mix (56.02\% / 24.02\% vs 58.13\% / 27.92\%), confirming that $c_{ab}$ starvation is the binding constraint. CovJac is the best method on all seven datasets; under a null of equal performance, a sign test gives $p = 2^{-7} < 0.01$ (Appendix~\ref{app:significance}).

\begin{table}[t]
\centering
\footnotesize
\caption{Cross-dataset comparison (last-10 accuracy $\pm$ std, 3 seeds;
best-checkpoint accuracy in Appendix Table~\ref{tab:best-acc}). \textbf{Bold} = best per column. $L{=}6$; $k$: MNIST 64k, Adult/Splice 4k, SVHN/CIFAR-10 128k, CIFAR-100 256k, MONK's-2 136 (100k iters for CIFAR-100; 10k for MONK's-2; rest 50k).}
\label{tab:mnist-main}
\scriptsize
\begin{tabular}{@{}lcccccccc@{}}
\toprule
Method & \#p/n & Adult & Splice & MONK's-2 & MNIST & SVHN & CIFAR-10 & CIFAR-100 \\
\midrule
Soft-Mix       & 16 & 84.68{\tiny$\pm$.07} & 97.01{\tiny$\pm$.11} & 81.32{\tiny$\pm$1.09} & 98.29{\tiny$\pm$.04} & 68.21{\tiny$\pm$.08} & 58.13{\tiny$\pm$.12} & 27.92{\tiny$\pm$.43} \\
Gumbel-ST      & 16 & 84.69{\tiny$\pm$.06} & 96.81{\tiny$\pm$.30} & 81.46{\tiny$\pm$.89} & 98.17{\tiny$\pm$.07} & 66.12{\tiny$\pm$.38} & 58.24{\tiny$\pm$.25} & 27.72{\tiny$\pm$.40} \\
M-STE (ours)   & 4  & 84.91{\tiny$\pm$.01} & 97.07{\tiny$\pm$.00} & 78.53{\tiny$\pm$1.25} & 98.09{\tiny$\pm$.02} & 66.69{\tiny$\pm$.11} & 56.02{\tiny$\pm$.23} & 24.02{\tiny$\pm$.17} \\
M-CovJac (ours)& 4  & \textbf{84.93}{\tiny$\pm$.03}$^\dagger$ & \textbf{97.20}{\tiny$\pm$.23} & \textbf{86.06}{\tiny$\pm$2.17} & \textbf{98.30}{\tiny$\pm$.04} & \textbf{68.91}{\tiny$\pm$.13} & \textbf{58.97}{\tiny$\pm$.26} & \textbf{28.37}{\tiny$\pm$.22} \\
\bottomrule
\end{tabular}

{\tiny $^\dagger$Adult: $\tau{=}0.1$ (all other datasets: $\tau{=}1.0$;
Appendix~\ref{app:tau}).}
\end{table}

\subsection{Depth stability}
\label{sec:depth}

\begin{figure}[t]
\centering
\begin{minipage}[t]{0.37\linewidth}
\vspace{0pt}
\centering
\includegraphics[width=\linewidth]{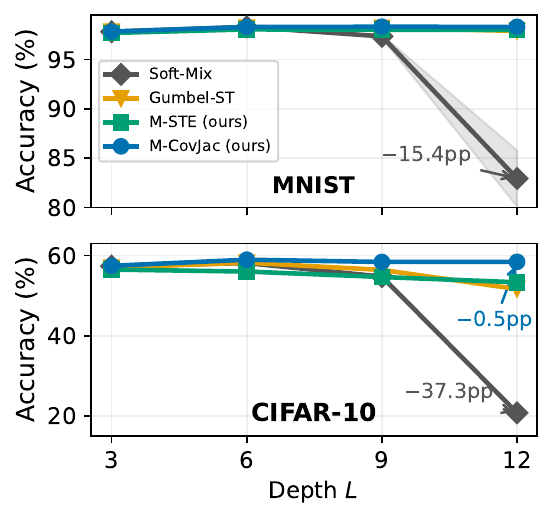}
\caption{Accuracy vs depth (3 seeds). Soft-Mix collapses on both
MNIST and CIFAR-10; CovJac holds ($-0.5$pp on CIFAR-10 at $L{=}12$).}
\label{fig:depth-scaling}
\end{minipage}%
\hfill
\begin{minipage}[t]{0.60\linewidth}
\vspace{0pt}
\centering
\footnotesize
\captionof{table}{Ablation on MNIST ($k{=}64\text{k}$, $L{=}6$, 3 seeds).
Each row differs from the one above by exactly one design choice. The +17.25pp IWP-STE$\to$M-STE gap isolates the basis effect (Appendix~\ref{app:basis-constraint}); the +0.21pp
M-STE$\to$M-CovJac gap isolates the gradient mechanism.
\textbf{Bold} = best. Last-10 accuracy.}
\label{tab:ablation}
\begin{tabular}{@{}lclr@{}}
\toprule
Method & Acc (\%) & Design change & $\Delta$ \\
\midrule
IWP-STE$^\dagger$    & 80.84$\pm$2.62 & corner + STE & -- \\
M-STE (ours) & 98.09$\pm$0.02 & $\to$ multilinear & +17.25 \\
M-CovJac (ours) & \textbf{98.30$\pm$0.04} & $\to$ CovJac & +0.21 \\
Soft-Mix             & 98.29$\pm$0.04 & $\to$ 16p softmax & $-$0.01 \\
\midrule
\textit{IWP$^\ddagger$}       & \textit{77.95$\pm$2.60} & \textit{continuous} & -- \\
\textit{M-free$^\S$}          & \textit{98.06$\pm$0.11} & \textit{continuous} & -- \\
\bottomrule
\end{tabular}

{\tiny $^\dagger$Our ablation. $^\ddagger$\citep{ruttgers2026narrowing}.
$^\S$\textit{Italic}: not Boolean gates.}
\end{minipage}
\end{figure}

\paragraph{Depth (Figure~\ref{fig:depth-scaling}).}
As predicted by Proposition~\ref{prop:softmix-cancel} and Observation~1, Soft-Mix collapses at depth because its gradient cancellation compounds across layers. On MNIST, Soft-Mix drops $-15.4$pp from $L{=}6$ to $L{=}12$; on CIFAR-10, the collapse reaches $-37.3$pp. In contrast, all single-gate-forward methods survive: CovJac is the most depth-stable ($-0.5$pp on CIFAR-10, essentially zero on MNIST), followed by STE and Gumbel-ST. On Adult with $\tau{=}0.1$, CovJac degrades only $-0.04$pp from $L{=}6$ to $L{=}12$ (Appendix Table~\ref{tab:adult-depth}). Doubling the training budget to 100k iterations partially recovers
Soft-Mix but still leaves a 6.0pp gap to single-gate methods
(Appendix Table~\ref{tab:depth-100k}), confirming that the penalty is structural, not a training-speed lag. Full depth tables for all datasets are in Appendix~\ref{app:depth-h6}.

\paragraph{Ablation (Table~\ref{tab:ablation}).}
\label{sec:ablation}
The +17.25pp IWP-STE $\to$ M-STE gap is purely from the basis change (both span the same function space; Appendix~\ref{app:cost}); the effect is even larger on CIFAR-10 (+22.65pp) and present on Adult (+8.50pp). The $2{\times}2$ basis$\times$constraint factorial (Appendix~\ref{app:basis-constraint}) confirms: the gap appears only under STE, not under continuous training. At $L{=}12$, Multilinear-free drops to 92.95\% while IWP-free holds at 97.80\% (Appendix Table~\ref{tab:basis-constraint-depth}), reversing the ranking: the multilinear basis's degree hierarchy becomes a liability without the snap constraint.

\subsection{Cost, width scaling, and interaction-demand scaling}
\label{sec:scaling}

\paragraph{Computational cost.}
At deployment all methods evaluate a single hard gate (${\sim}7$ ops).
M-STE is $25\times$ cheaper at training (one polynomial vs 16); all
4-parameter methods save $4\times$ optimizer memory (Appendix Table~\ref{tab:cost}).

\paragraph{Width scaling (empirical).}
At $L{=}1$ on MNIST (isolating the per-layer effect), CovJac achieves 61\% faster error decay with width: fitting $A(k) = 100 - Bk^{-\alpha}$ (3 seeds) gives $\alpha{=}0.42$ vs Soft-Mix's $0.26$, likely because the $J_{03}$ coupling makes more neurons effectively trainable per width unit. At matched parameter budget ($4\times$ wider CovJac), the gap is $+3.3$pp, costing only $1/6$ the inference FLOPs (Appendix~\ref{app:width-scaling}).

\paragraph{Interaction-demand scaling.}
The CovJac advantage over STE scales monotonically with task interaction demand: from $+0.02$pp on separable Adult to $+7.53$pp on interaction-dense MONK's-2 (Appendix Figure~\ref{fig:interaction-scaling}). The intuition is that STE's 25\% $c_{ab}$ coverage (Lemma~\ref{lem:coverage}) is sufficient when a task can be solved with mostly separable gates, but becomes the binding constraint when interactive gates (AND, XOR) are needed. We formalize this as the \emph{$c_{ab}$ starvation severity} (the Soft-Mix$\,-\,$STE accuracy deficit) and find Pearson $r{=}0.85$ correlation with the CovJac advantage (Appendix~\ref{app:interaction-demand}, Table~\ref{tab:starvation-severity}). The advantage persists at larger scale: on CIFAR-100 at $k{=}1280\text{k}$ and Tiny-ImageNet at $k{=}2560\text{k}$, CovJac remains the best method with growing gaps (Appendix~\ref{app:scale-up}).

\subsection{Diagnostic experiments}
\label{sec:diagnostics}

\begin{figure}[t]
\centering
\includegraphics[width=0.88\linewidth]{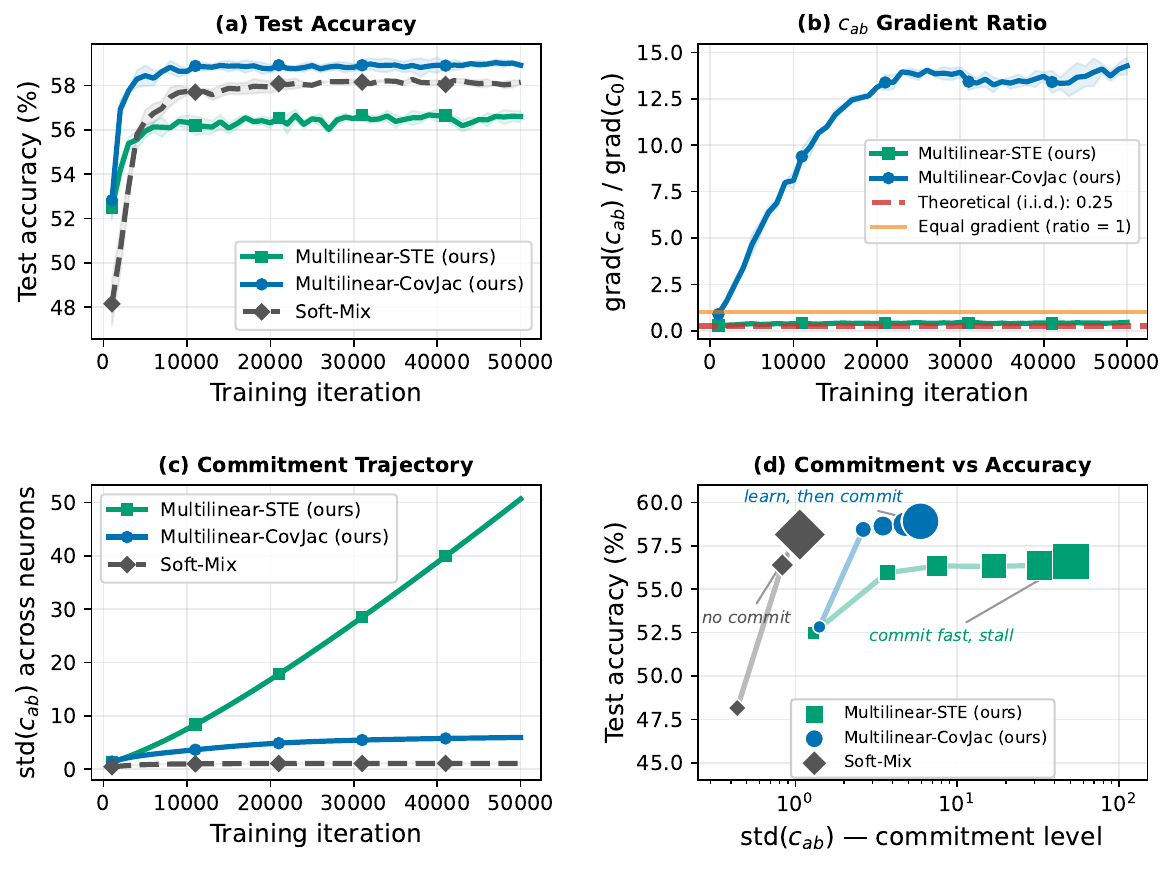}
\caption{CIFAR-10 ($k{=}128\text{k}$, $L{=}6$, 3 seeds).
\textbf{(a)}~Accuracy over training. \textbf{(b)}~$c_{ab}$ gradient ratio: CovJac reaches $14\times$, STE stays at theoretical 0.25. \textbf{(c)}~Commitment ($\mathrm{std}(c_{ab})$): STE commits fast, CovJac slow, Soft-Mix never. \textbf{(d)}~Commitment vs accuracy: STE commits to wrong gates and stalls; CovJac learns first, then commits; Soft-Mix improves without committing.}
\label{fig:commitment}
\end{figure}

\paragraph{Commitment dynamics (Figure~\ref{fig:commitment}).}
Figure~\ref{fig:commitment} tracks three quantities on CIFAR-10 over training. Panel~(b) shows the gradient-norm ratio $\|\nabla c_{ab}\|/\|\nabla c_0\|$: for STE this stabilizes at 0.227 (matching the theoretical $\E[ab]{=}1/4$ from Lemma~\ref{lem:coverage}), while for CovJac it reaches 14.27, confirming that the $J_{03}$ coupling (Proposition~\ref{prop:covjac}) amplifies the $c_{ab}$ gradient by ${\sim}60\times$ relative to STE. Panel~(c) measures gate commitment via $\mathrm{std}(c_{ab})$ across neurons: STE commits quickly (std=50.7) but to suboptimal gates (56.02\% final accuracy); CovJac commits slowly (std=5.9) but accurately (58.97\%); Soft-Mix never commits (std=1.1). Panel~(d) reveals the mechanism: STE locks into gates early before learning which ones are correct, while CovJac explores first and commits only after accuracy plateaus (Appendix~\ref{app:gradnorm}).

\paragraph{Signal survival (Figure~\ref{fig:cancellation}).}
To directly test Proposition~\ref{prop:softmix-cancel}, we measure per-layer signal survival $s_l = \|\partial z_l/\partial\mathbf{h}_{l-1}\|/\|\partial z_l/\partial\mathbf{h}_{l-1}\|_{\max}$, the fraction of gradient signal not cancelled. At initialization on MNIST ($L{=}12$, logit std $\sigma{=}1.0$),
Soft-Mix has $s_l \approx 0.29$ (71\% of gradient cancelled at each layer), while M-STE has $s_l = 1.0$ and CovJac has $s_l = 0.993$. This directly confirms the theory: Soft-Mix cancels because it sums over all 16 gates (Proposition~\ref{prop:softmix-cancel}); single-gate methods do not (Proposition~\ref{prop:ste-noncancel}).

\paragraph{Gate selection at convergence
(Figure~\ref{fig:codebook-neurons}).} The gate type distribution at convergence reveals how each method uses the 16-gate codebook. CovJac selects the highest fraction of strong-interaction gates ($|c_{ab}|{\geq}2$): 25.7\% on CIFAR-10, 27.2\% on MNIST, 36.2\% on Adult, compared to 16--21\% for all other methods. CovJac also has the fewest constant gates ($<$0.5\% vs 4--10\%), indicating almost no wasted neurons. This is a direct consequence of the $J_{03}$ coupling: gradient flows to $c_{ab}$ on every sample, pushing neurons toward interactive gates (Appendix~\ref{app:codebook-neurons}).

\begin{figure}[tb]
\centering
\begin{minipage}[t]{0.35\linewidth}
\vspace{-15pt}
\centering
\includegraphics[width=\linewidth]{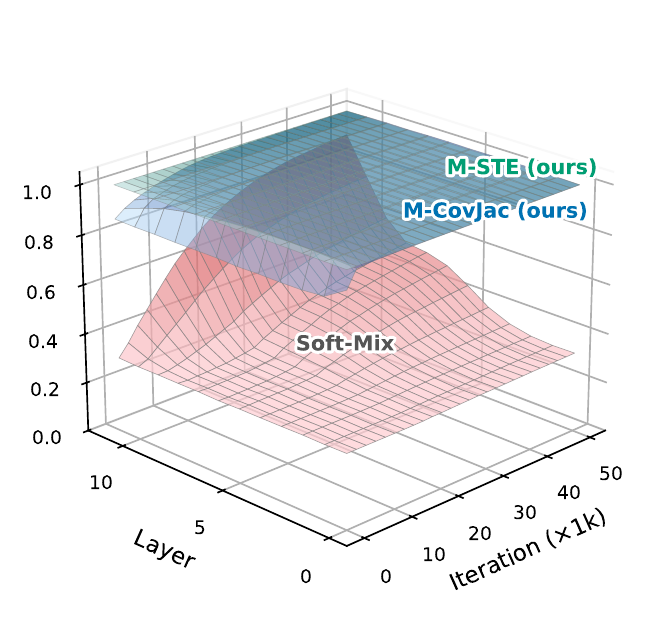}
\caption{Signal survival $s_l$ per layer on
MNIST ($L{=}12$, 3 seeds). Soft-Mix (\textcolor{red}{red}): $s_l\!\approx\!0.3$ (71\% cancelled), confirming Proposition~\ref{prop:softmix-cancel}. Ours (\textcolor{blue}{blue}, \textcolor{teal}{teal}): $s_l\!\approx\!1.0$ (no cancellation).}
\label{fig:cancellation}
\end{minipage}%
\hfill
\begin{minipage}[t]{0.63\linewidth}
\vspace{0pt}
\centering
\includegraphics[width=\linewidth]{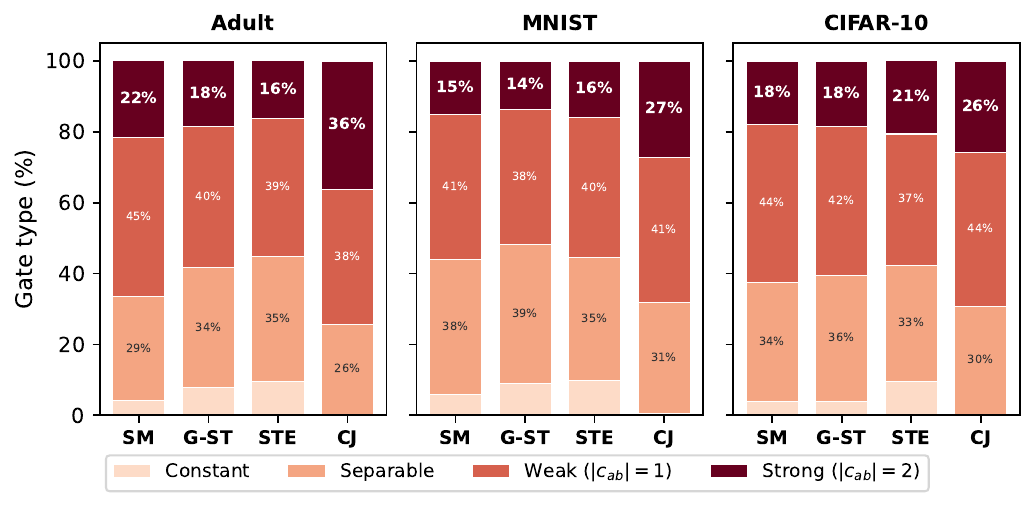}
\caption{Gate type distribution (3 seeds).
SM=Soft-Mix, G-ST=Gumbel-ST, STE=M-STE (ours), CJ=M-CovJac (ours). CovJac: highest strong-interaction fraction (darkest, 26--36\%), fewest constant ($<$0.5\%, not visible in bars).}
\label{fig:codebook-neurons}
\end{minipage}
\end{figure}

\paragraph{Controls.}
Hard-ST (16p, argmax + STE) confirms single-gate-forward suffices for depth stability (98.11\% MNIST, 57.39\% CIFAR-10) but does not fix $c_{ab}$ starvation: it has 16 parameters and still suffers the 25\% coverage limitation. Walsh and smoothed bases also fail, as predicted by Proposition~\ref{prop:basis-futility}: Walsh achieves full coverage ($\rho{=}1$) but exact coherence cancellation ($\kappa{=}0$, 94.92\%); the smoothed basis has $\rho{=}1$ and $\kappa{=}1$ but non-zero STE bias ($B{>}0$), collapsing to 10.14\% (Appendix~\ref{app:ste-failures}).


\section{Conclusion}
\label{sec:conclusion}

The 16 two-input Boolean gates are 16 points in a 4-dimensional polynomial space. Training in this space rather than the 16-dimensional softmax eliminates 11 nullspace gradient directions, avoids exact gradient cancellation at uniform initialization, and reduces each neuron from 16 to 4 parameters. The CovJac gradient mechanism further bypasses interaction-coefficient starvation by coupling $c_{ab}$ to the always-active constant channel. Empirically, CovJac matches or exceeds Soft-Mix on all seven datasets using $4\times$ fewer parameters. The CovJac-vs-STE gap grows monotonically with task interaction demand ($r{=}0.85$ correlation with starvation severity; Appendix~\ref{app:interaction-demand}), from $+0.02$pp on separable Adult to $+7.53$pp on interaction-dense MONK's-2. At depth, Soft-Mix collapses ($-37.3$pp on CIFAR-10 at $L{=}12$) while CovJac holds ($-0.5$pp on CIFAR-10), directly confirming the backward-cancellation theory. The three pathologies identified here arise from structural properties of the codebook, not from Boolean gates specifically: nullspace redundancy appears whenever a codebook has rank less than its size; backward cancellation appears under complementation symmetry; and starvation appears whenever monomial-degree ordering creates hierarchical coverage under STE. The CovJac coupling (Proposition~\ref{prop:covjac}) applies to any fixed codebook. We demonstrate these effects on 2-input Boolean gates; whether they arise in other codebook-based selection problems (NAS operation sets, MoE expert assignments) is an empirical question for future work, alongside $k$-input gates and learnable input wiring (Appendix~\ref{app:wiring-sensitivity}).

\section*{Acknowledgments}

The author thanks Oliver and Guy at StarAI Lab, UCLA, for constructive discussions on (probabilistic) logic circuits and the lab members (Benjie, Poorva, Daniel, Renato, Zoe, Gwen) for their kind interactions during a research visit.

\bibliographystyle{unsrtnat}
\bibliography{references}

\appendix

\section{Full Proofs}
\label{app:proofs}

\paragraph{Notation.}
For each of the 16 Boolean functions $g_j : \{0,1\}^2 \to \{0,1\}$ (indexed $j = 0,\ldots,15$ by the truth-table bit pattern), we write its unique multilinear representation \[ g_j(a,b) = c_{0,j} + c_{a,j}\, a + c_{b,j}\, b + c_{ab,j}\, ab, \] with coefficient vector $\bG_j = (c_{0,j},\, c_{a,j},\, c_{b,j},\, c_{ab,j})^\top \in \Z^4$. Collect these as rows of the codebook $\bG\in\Z^{16\times 4}$. The canonical polynomial basis is $\bpsi(a,b) = (1, a, b, ab)^\top$. A continuous learned coefficient vector is denoted $\bc\in\R^4$, and under Multilinear-STE the forward ``hard'' gate is $\bG_{\text{hard}} = \bG_{j^*}$ with $j^* = \argmin_j \|\bc-\bG_j\|^2$. Under Multilinear-CovJac the soft coefficient is $\bc_{\text{soft}} = \sum_j \omega_j \bG_j$ with $\omega_j = \mathrm{softmax}(-\|\bc-\bG_j\|^2/\tau)_j$ (we write $\omega_j$ for the CovJac proximity weights and reserve $\pi_j$ for the
Soft-Mix gate distribution, following \S\ref{sec:method}).

\subsection{Zero Discretization Gap (Identity~1)}
\label{app:identity1}

\paragraph{Identity~1 (Polynomial collapse)}
Every function $g:\{0,1\}^2 \to \R$ admits a unique multilinear polynomial representation \[ g(a,b) = c_0 + c_a\, a + c_b\, b + c_{ab}\, ab, \] with coefficients \[ c_0 = g(0,0),\ c_a = g(1,0)-g(0,0),\ c_b = g(0,1)-g(0,0),\ c_{ab} = g(1,1)-g(1,0)-g(0,1)+g(0,0). \] For the 16 Boolean-valued functions $g : \{0,1\}^2 \to \{0,1\}$, the resulting coefficient vectors $\bG_j = (c_0, c_a, c_b, c_{ab})^\top$ lie in $\Z^4$, with entries in the ranges $c_0 \in \{0,1\}$, $c_a, c_b \in \{-1,0,1\}$, $c_{ab} \in \{-2,-1,0,1,2\}$.

\emph{Proof.} (Direct corner-value expansion / dimension argument.) Let $V$ be the vector space of functions $\{0,1\}^2 \to \R$, which has dimension 4 since a function is determined by its four corner values $g(0,0), g(1,0), g(0,1), g(1,1)$. Consider the four monomials $\{1,\,a,\,b,\,ab\}$ as elements of $V$. Their corner-value vectors are \[ 1 \leftrightarrow (1,1,1,1),\ a \leftrightarrow (0,1,0,1),\ b \leftrightarrow (0,0,1,1),\ ab \leftrightarrow (0,0,0,1), \] where entries are ordered $(a,b) = (0,0),(1,0),(0,1),(1,1)$. The $4\times 4$ matrix formed by these vectors, \[ M = \begin{pmatrix}1 & 0 & 0 & 0\\ 1 & 1 & 0 & 0\\ 1 & 0 & 1 & 0\\ 1 & 1 & 1 & 1\end{pmatrix}, \] is lower triangular with unit diagonal, hence $\det M = 1 \neq 0$. The four monomials are therefore linearly independent in $V$, and since $\dim V = 4$, they form a basis. Uniqueness of the representation is immediate from basis uniqueness.

To obtain the explicit formulae, solve $M\bc = (g(0,0),g(1,0),g(0,1),g(1,1))^\top$ by forward substitution:
\begin{itemize}
\item Row 1: $c_0 = g(0,0)$.
\item Row 2: $c_0 + c_a = g(1,0)\ \Rightarrow\ c_a = g(1,0) - g(0,0)$.
\item Row 3: $c_0 + c_b = g(0,1)\ \Rightarrow\ c_b = g(0,1) - g(0,0)$.
\item Row 4: $c_0 + c_a + c_b + c_{ab} = g(1,1)\ \Rightarrow\ c_{ab} = g(1,1) - g(1,0) - g(0,1) + g(0,0)$.
\end{itemize}
For Boolean-valued $g$ (i.e.\ $g \in \{0,1\}$), the four corner values are each in $\{0,1\}$, so $c_0 \in \{0,1\}$, $c_a, c_b \in \{-1,0,1\}$, and $c_{ab} \in \{-2,-1,0,1,2\}$. Enumerating the $2^4 = 16$ Boolean-valued corner assignments gives 16 distinct integer coefficient vectors, one per Boolean gate. \hfill$\square$

\emph{Attribution:}~\citep{hammer1968boolean,owen1972multilinear,crama2011boolean}. The $\{-1,+1\}$ Fourier reformulation is in \citet[Proposition~1.1]{odonnell2014analysis}. Elementary; proved here for self-containedness.

\paragraph{Zero discretization gap.}
For any Multilinear-STE layer with binary inputs $(a,b) \in \{0,1\}^2$, \[ z_{\text{train}}(a,b) = z_{\text{eval}}(a,b). \]

\emph{Proof.} Let $\bch = \bG_{j^*}$ be the committed Boolean gate (the codebook entry nearest to $\bc$ in squared Euclidean distance, ties broken by lexicographic order). The training-time forward pass of Multilinear-STE emits \[ z_{\text{train}}(a,b) = \bch^\top \bpsi(a,b) = \bG_{j^*}^\top \bpsi(a,b). \] The evaluation-time forward pass emits the Boolean gate value itself, $z_{\text{eval}}(a,b) = g_{j^*}(a,b)$. By Identity~1, the multilinear polynomial with coefficient vector $\bG_{j^*}$ equals $g_{j^*}$ on $\{0,1\}^2$, i.e.\ precisely $g_{j^*}(a,b) = \bG_{j^*}^\top \bpsi(a,b)$ for all $(a,b)\in\{0,1\}^2$. Substituting, $z_{\text{eval}}(a,b) = \bG_{j^*}^\top \bpsi(a,b) = z_{\text{train}}(a,b)$.
\hfill$\square$

\paragraph{Remark.}
Multilinear-CovJac does not satisfy the zero-gap property as a pointwise identity (since $\bc_{\text{soft}} \neq \bG_{j^*}$ during training), yet empirically the discretization gap is ${\leq}\,0.2\%$ on all main datasets because the trained $\bpi$ concentrates on $\bG_{j^*}$ strongly enough that $\bc_{\text{soft}}^\top\bpsi(a,b)$ rounds to $g_{j^*}(a,b)$ at every Boolean input.

\subsection{Gradient Utilization (Lemma~\ref{lem:coverage})}
\label{app:coverage}

\begin{lemma}[Gradient coverage hierarchy]
\label{lem:coverage}
Under i.i.d.\ Bernoulli$(1/2)$ binary inputs $(a,b)\in\{0,1\}^2$, \[ \E\bigl[\|\bpsi(a,b)\|_0\bigr] = 2.25 \quad \text{(Multilinear / polynomial basis)}, \] \[ \E\bigl[\|\boldsymbol{\phi}(a,b)\|_0\bigr] = 1.0 \quad \text{(Corner / IWP basis)}. \]
\end{lemma}

\emph{Proof.} (Component-wise expectation.)

\textbf{Polynomial basis.} We have $\bpsi(a,b) = (1, a, b, ab)^\top$. The $\ell_0$ norm counts non-zero entries, so by linearity of expectation \[ \E\bigl[\|\bpsi(a,b)\|_0\bigr] = \E[\mathbf{1}_{1\neq 0}] + \E[\mathbf{1}_{a\neq 0}] + \E[\mathbf{1}_{b\neq 0}] + \E[\mathbf{1}_{ab\neq 0}]. \] Compute each term: $\E[\mathbf{1}_{1\neq 0}] = 1$; $\E[\mathbf{1}_{a\neq 0}] = \Pr[a=1] = 1/2$; $\E[\mathbf{1}_{b\neq 0}] = 1/2$; $\E[\mathbf{1}_{ab\neq 0}] = \Pr[a=1 \wedge b=1] = 1/4$ (independence). Sum: $1 + 1/2 + 1/2 + 1/4 = 9/4 = 2.25$.

\textbf{Corner basis.} The corner basis $\boldsymbol{\phi}(a,b)$ is the one-hot indicator of $(a,b)$ over the four corners of $\{0,1\}^2$, i.e.\ $\phi_i(a,b) = \mathbf{1}_{(a,b)=v_i}$ where $v_i$ ranges over $(0,0),(1,0),(0,1),(1,1)$. For any input $(a,b)$, exactly one component of $\boldsymbol{\phi}$ equals 1 and the others equal 0, so $\|\boldsymbol{\phi}(a,b)\|_0 = 1$ deterministically. Hence $\E\|\boldsymbol{\phi}\|_0 = 1$. \hfill$\square$

\paragraph{Interpretation.}
The polynomial basis distributes the learning load across degrees: the degree-$d$ monomial has coverage $(1/2)^d$. The four components $c_0, c_a, c_b, c_{ab}$ receive gradient on $100\%, 50\%, 50\%, 25\%$ of samples respectively, giving the natural degree-ordered curriculum discussed in Section~\ref{sec:method}. The corner basis, by contrast, has a flat $25\%$ coverage per coordinate but only \emph{one} non-zero coordinate per sample, yielding the same 1.0 expected $\ell_0$ norm.

\subsection{Backward Gradient Cancellation in Soft-Mix (Lemma~\ref{lem:zerosum}, Proposition~\ref{prop:softmix-cancel}, Observation~1)}
\label{app:cancellation}

\paragraph{Lemma~\ref{lem:zerosum} (Zero-sum gate derivative symmetry)}
\[ \sum_{j=0}^{15} c_{a,j} \;=\; \sum_{j=0}^{15} c_{b,j} \;=\; \sum_{j=0}^{15} c_{ab,j} \;=\; 0. \] Consequently, for every $(a,b)\in[0,1]^2$, \[ \sum_{j=0}^{15} \frac{\partial g_j}{\partial a}(a,b) \;=\; \sum_{j=0}^{15} (c_{a,j} + c_{ab,j}\, b) \;=\; 0, \] and symmetrically for $\partial_b$.

\emph{Proof.} (Complementation pairing.) The set $\mathcal{B}$ of Boolean functions $\{0,1\}^2 \to \{0,1\}$ is closed under pointwise complementation $g \mapsto \bar g := 1 - g$. This map is an involution on $\mathcal{B}$ ($\bar{\bar g} = g$) and has no fixed points: $g = \bar g$ would imply $g \equiv 1/2$, which is not Boolean-valued. Hence $\mathcal{B}$ partitions into $16/2 = 8$ complementary pairs $(g_j, g_{\bar j})$.

Compute the multilinear coefficients of $\bar g_j$. Using the identity $1 - g_j(a,b) = 1 - c_{0,j} - c_{a,j}a - c_{b,j}b - c_{ab,j}ab$: \[ \bar g_j(a,b) = (1 - c_{0,j}) + (-c_{a,j})\, a + (-c_{b,j})\, b + (-c_{ab,j})\, ab. \] By uniqueness of the multilinear representation (Identity~1) applied to $\bar g_j$, the coefficient vector of $\bar g_j$ is $(1 - c_{0,j},\, -c_{a,j},\, -c_{b,j},\, -c_{ab,j})$. In particular, the three \emph{non-constant} coefficients satisfy $c_{a,\bar j} = -c_{a,j}$, $c_{b,\bar j} = -c_{b,j}$, $c_{ab,\bar j} = -c_{ab,j}$, so that $c_{a,j} + c_{a,\bar j} = 0$, and likewise for $c_b$ and $c_{ab}$.

Summing over all 16 gates by grouping them into the 8 complementary pairs, \[ \sum_{j=0}^{15} c_{a,j} \;=\; \sum_{\text{pairs } (j,\bar j)} (c_{a,j} + c_{a,\bar j}) \;=\; 0. \] The same argument gives $\sum_j c_{b,j} = 0$ and $\sum_j c_{ab,j} = 0$. The partial derivative identity is then immediate: \[ \sum_{j=0}^{15} \frac{\partial g_j}{\partial a}(a,b) = \Bigl(\sum_j c_{a,j}\Bigr) + b\Bigl(\sum_j c_{ab,j}\Bigr) = 0. \]
\hfill$\square$

\emph{Attribution.} Elementary from the complementation symmetry of Boolean functions; the closely related Fourier-basis zero-sum identity is in \citet{odonnell2014analysis}. We do not claim it as a contribution. The novelty is its application in Proposition~\ref{prop:softmix-cancel}.

\paragraph{Proposition~\ref{prop:softmix-cancel} (Soft-Mix backward gradient vanishes at uniform $\bpi$)}
Let $z(a,b) = \sum_{j=0}^{15} \pi_j\, g_j(a,b)$ for $\bpi \in \Delta^{15}$. If $\pi_j = 1/16$ for all $j$, then $\partial z/\partial a(a,b) = \partial z/\partial b(a,b) = 0$ for every $(a,b)\in[0,1]^2$. Hence for any scalar loss $\mathcal{L}$, $\partial \mathcal{L}/\partial a |_{\text{through } z} = 0$.

\emph{Proof.} Substituting the multilinear form into $z$ and differentiating, \[ \frac{\partial z}{\partial a}(a,b) = \sum_{j=0}^{15} \pi_j (c_{a,j} + c_{ab,j} b). \] At $\pi_j = 1/16$, this equals $(1/16)[\sum_j c_{a,j} + b\sum_j c_{ab,j}] = 0$ by Lemma~\ref{lem:zerosum}. This holds for every $(a,b)\in[0,1]^2$ because the algebraic cancellation is independent of $a$ and holds for every fixed $b$. The $\partial z/\partial b$ statement follows by the identical computation with $a,b$ swapped. The loss-gradient statement is the chain rule.
\hfill$\square$

\paragraph{Near-uniform regime.}
If $\bpi$ deviates from uniform $\mathbf{u} = (1/16,\ldots,1/16)$, write $\pi_j = 1/16 + \delta_j$ with $\sum_j \delta_j = 0$. Then \[ \frac{\partial z}{\partial a} = \sum_j \delta_j (c_{a,j} + c_{ab,j} b), \] since the $1/16$-term vanishes by Lemma~\ref{lem:zerosum}. On Boolean inputs $b\in\{0,1\}$, every gate satisfies $|c_{a,j}+c_{ab,j}b|\leq 1$ (verified by enumeration, Table~\ref{tab:ste-enumeration}). Hölder's inequality gives \[ \left|\frac{\partial z}{\partial a}\right| \;\le\; \|\bpi - \mathbf{u}\|_1, \] so the backward signal grows only as $\bpi$ departs from uniformity. This will be important for Observation~1.

\paragraph{Connection to~\citet{petersen2022deep}.}
The forward-pass analogue of this identity is $z(a,b) = \tfrac{1}{16}\sum_j g_j(a,b) = 1/2$ for all $(a,b)$, because exactly 8 of the 16 gates output 1 at each corner. Both cancellations share the same root: complementation symmetry.

\paragraph{Observation~1 (Dynamical depth attenuation).}
In a DLGN without residual connections, trained with Soft-Mix from near-uniform gate-logit initialization:
\begin{enumerate}
\item At initialization, every neuron has $\bpi^{(l)}$ near uniform, so
by the near-uniform bound after Proposition~\ref{prop:softmix-cancel}, its input Jacobian $\partial z^{(l)}/\partial(a^{(l)},b^{(l)})$ has small norm.
\item A layer $l$ can commit (move its $\bpi^{(l)}$ away from uniform)
only if it receives a non-trivial $\bpi$-gradient, which requires $\partial \mathcal{L}/\partial z^{(l)}$ to be non-trivial, which in turn requires layers $l+1,\ldots,L$ to already have non-trivial input Jacobians.
\item Hence commitment proceeds greedily from the output layer toward
the input layer. At sufficient depth, early layers may never commit within a finite training budget.
\end{enumerate}

\paragraph{Assumptions used.} (a) No residual/skip connections; (b)
near-uniform initialization of gate logits; (c) the analysis is dynamical, not a static worst-case bound; (d) the gradient path of interest is the loss-to-input path through layer input Jacobians.

\paragraph{What we do NOT claim.}
We do \emph{not} claim a bound of the form $|\partial\mathcal{L}/\partial x| \le C\varepsilon^{L-l}$ as a theorem, because the hypothesis ``all intermediate layers are simultaneously near-uniform throughout training'' is effectively the conclusion restated. Proposition~\ref{prop:softmix-cancel} is the per-layer pointwise fact; Observation~1 is the dynamical interpretation; Section~\ref{sec:depth} is the empirical confirmation. Together they explain the depth penalty; none is a worst-case depth-decay theorem.

\paragraph{Single-gate-forward corollary (remark).}
The cancellation argument in Proposition~\ref{prop:softmix-cancel} is specific to the forward expression $z = \sum_{j=0}^{15}\pi_j g_j(a,b)$: it is only because $\partial z/\partial a = \sum_j \pi_j(c_{a,j}+c_{ab,j}b)$ contains the Lemma~\ref{lem:zerosum} zero-sums inside a 16-term sum that those zero-sums can collapse the derivative to 0 at uniform $\bpi$. Any forward pass that commits to a \emph{single} gate $j^\star$, writing $z = g_{j^\star}(a,b) = c_{0,j^\star} + c_{a,j^\star}a + c_{b,j^\star}b + c_{ab,j^\star}ab$, gives $\partial z/\partial a = c_{a,j^\star} + c_{ab,j^\star}b$, a derivative of a single committed polynomial with no sum in which to cancel; Proposition~\ref{prop:ste-noncancel} then guarantees magnitude $\geq 1$ for every $a$-dependent gate. This corollary applies uniformly to: (i) Multilinear-STE (hard $L_2$ snap; Proposition~\ref{prop:ste-noncancel}); (ii) Multilinear-CovJac (soft-VQ with $\bc_{\text{soft}} = \sum_j\omega_j\bG_j$: the forward pass evaluates a single polynomial $z = \bc_{\text{soft}}^\top\bpsi(a,b)$, not a mixture over 16 polynomials); (iii) Gumbel-ST (Gumbel noise + argmax samples a one-hot $\bpi = e_{j^\star}$ per forward, reducing to case (i) with a stochastic $j^\star$). Soft-Mix is the only method in the comparison whose forward \emph{does} compute a full 16-term sum, and is therefore the only one to which Proposition~\ref{prop:softmix-cancel} applies. The depth sweep (Appendix~\ref{app:depth-h6}) confirms this cleanly: at $L{=}12$, $k{=}64\text{k}$, Multilinear-STE (98.03\%), Multilinear-CovJac (98.23\%), and Gumbel-ST (97.81\%) all remain above 97.8\%, while Soft-Mix stalls at 84.92\%. Gumbel-ST has the same 16-parameter count as Soft-Mix, so the dichotomy is about forward-pass structure, not parameter count. This is an empirical corollary of the proof structure above, not a separate theorem.

\subsection{Multilinear-STE and CovJac Eliminate Cancellation (Propositions~\ref{prop:ste-noncancel}, \ref{prop:covjac})}
\label{app:ste-covjac}

\paragraph{Proposition~\ref{prop:ste-noncancel} (Multilinear-STE backward non-cancellation)}
In Multilinear-STE, the forward pass at a neuron with committed gate $\bch = (c_0,c_a,c_b,c_{ab})$ is $z = \bch^\top\bpsi(a,b)$, and the STE backward derivative is $\partial z/\partial a = c_a + c_{ab}\, b$. For every non-trivial Boolean gate depending on input $a$ (i.e.\ every Boolean gate with $(c_a, c_{ab}) \neq (0,0)$), $\max_{b\in\{0,1\}} |c_a + c_{ab}\, b| \;\ge\; 1$.

\emph{Proof.} The STE rule treats $\bch$ as if it were the free parameter for the backward pass, so $\partial z/\partial a = c_a + c_{ab}\, b$. Since $\bch$ is a Boolean-gate coefficient vector, Identity~1 gives $c_a \in \{-1,0,1\}$ and $c_{ab} \in \{-2,-1,0,1,2\}$. A Boolean gate depends on $a$ iff $c_a \neq 0$ or $c_{ab} \neq 0$.

\textbf{Case 1: $c_a \neq 0$.} At $b = 0$, $|c_a + c_{ab}\cdot 0| = |c_a| \ge 1$.

\textbf{Case 2: $c_a = 0$ and $c_{ab} \neq 0$.} At $b = 1$, $|c_a + c_{ab}\cdot 1| = |c_{ab}| \ge 1$.

These two cases exhaust all gates depending on $a$. \hfill$\square$

\paragraph{Enumeration.}
For completeness, here are $(c_a,c_{ab})$ and $(\partial z/\partial a|_{b=0},\ \partial z/\partial a|_{b=1})$ for the 12 Boolean gates that depend on $a$:

\begin{table}[t]
\centering
\small
\caption{STE backward derivative $\partial z/\partial a = c_a + c_{ab}b$ for all 12 $a$-dependent Boolean gates. Every row satisfies $\max_{b\in\{0,1\}}|c_a+c_{ab}b|\geq 1$.}
\label{tab:ste-enumeration}
\begin{tabular}{lrrrr}
\toprule
Gate & $c_a$ & $c_{ab}$ & at $b{=}0$ & at $b{=}1$ \\
\midrule
$a$               & 1 & 0 & 1 & 1 \\ $\neg a$          & $-1$ & 0 & $-1$ & $-1$ \\ AND $ab$          & 0 & 1 & 0 & 1 \\ NAND $1-ab$       & 0 & $-1$ & 0 & $-1$ \\ OR $a+b-ab$       & 1 & $-1$ & 1 & 0 \\ NOR $1-a-b+ab$    & $-1$ & 1 & $-1$ & 0 \\ XOR $a+b-2ab$     & 1 & $-2$ & 1 & $-1$ \\ XNOR $1-a-b+2ab$  & $-1$ & 2 & $-1$ & 1 \\ $a\wedge\neg b$   & 1 & $-1$ & 1 & 0 \\ $\neg a\wedge b$  & 0 & $-1$ & 0 & $-1$ \\ $a\vee\neg b$     & 0 & 1 & 0 & 1 \\ $\neg a\vee b$    & $-1$ & 1 & $-1$ & 0 \\
\bottomrule
\end{tabular}
\end{table}

Every row has $\max_b|\cdot| \ge 1$, confirming the bound.

\paragraph{Key structural contrast with Soft-Mix.}
The Multilinear-STE backward derivative is the derivative of \emph{one} committed polynomial, not a convex combination of 16 gate derivatives. The zero-sum of Lemma~\ref{lem:zerosum} therefore has no avenue to manifest. The per-layer signal magnitude is $O(1)$, and a depth-$L$ product of such factors is $O(1)^{L} = O(1)$, not exponentially decaying.

\paragraph{Pre-commitment remark.}
During training, the learned $\bc$ is continuous and may lie near a Voronoi boundary, in which case the STE derivative $c_a + c_{ab}b$ (using the \emph{continuous} $c_a,c_{ab}$) can be small. The content of Proposition~\ref{prop:ste-noncancel} is structural: unlike Soft-Mix, \emph{structural cancellation across gates cannot occur}, so any smallness of the backward signal is a property of a single neuron's continuous parameter state, not a systematic symmetry that affects all neurons at initialization.

\paragraph{Proposition~\ref{prop:covjac} (CovJac selection Jacobian: input-independence, PSD, positive diagonal)}
For the soft-selection map $\bc\mapsto \bc_{\text{soft}}(\bc) = \sum_j \omega_j(\bc)\bG_j$, with $\omega_j(\bc) = \mathrm{softmax}(-\|\bc-\bG_j\|^2/\tau)_j$, the Jacobian $J_{ik}(\bc) = \partial(\bc_{\text{soft}})_i/\partial c_k$ equals \[ J_{ik} = \frac{2}{\tau}\,\mathrm{Cov}_\omega(G_{\cdot i},\, G_{\cdot k}) = \frac{2}{\tau}\Bigl(\sum_{j}\omega_j G_{ji}G_{jk} - \bigl(\sum_j \omega_j G_{ji}\bigr)\bigl(\sum_j \omega_j G_{jk}\bigr)\Bigr), \] and satisfies (i) input-independence, (ii) symmetric PSD, (iii) strictly positive diagonal for every coordinate at $\tau{>}0$ on the 16-Boolean codebook.

\emph{Proof.} Throughout this proof, $\omega_j$ denotes the CovJac proximity weights from Eq.~\ref{eq:softvq}.

\textbf{(a) Derivation of $\mathbf{J}$.} Let $d_j(\bc) = \|\bc-\bG_j\|^2$. Then $\omega_j = e^{-d_j/\tau}/Z$ where $Z = \sum_\ell e^{-d_\ell/\tau}$. Compute $\partial\omega_j/\partial c_k$ via the softmax chain rule. First, $\partial d_j/\partial c_k = 2(c_k - G_{jk})$. By the standard softmax derivative ($\partial\omega_j/\partial l_m = \omega_j(\delta_{jm}-\omega_m)$ applied to logits $l_j = -d_j/\tau$ via the chain rule), \[ \frac{\partial \omega_j}{\partial c_k} = -\frac{1}{\tau}\omega_j\Bigl(\frac{\partial d_j}{\partial c_k} - \sum_\ell \omega_\ell\frac{\partial d_\ell}{\partial c_k}\Bigr). \] Substituting $\partial d_j/\partial c_k = 2(c_k - G_{jk})$:
\begin{align*}
\frac{\partial \omega_j}{\partial c_k}
&= -\frac{2}{\tau}\omega_j\Bigl[
(c_k - G_{jk}) - \sum_\ell \omega_\ell(c_k - G_{\ell k})\Bigr] \\
&= -\frac{2}{\tau}\omega_j\Bigl[
\cancel{c_k} - G_{jk} - \cancel{c_k}\underbrace{\textstyle\sum_\ell\omega_\ell}_{=\,1} + \underbrace{\textstyle\sum_\ell\omega_\ell G_{\ell k}}_{=\,(\bc_{\text{soft}})_k} \Bigr] \\
&= \frac{2}{\tau}\omega_j\bigl(G_{jk} - (\bc_{\text{soft}})_k\bigr).
\tag{$\ast$}
\end{align*}
Now differentiate $(\bc_{\text{soft}})_i = \sum_j \omega_j G_{ji}$ by $c_k$: \[ J_{ik} = \sum_j \frac{\partial \omega_j}{\partial c_k}\, G_{ji} \overset{(\ast)}{=} \frac{2}{\tau}\sum_j \omega_j\bigl(G_{jk} - (\bc_{\text{soft}})_k\bigr)G_{ji}. \] Expanding $G_{jk} - (\bc_{\text{soft}})_k = G_{jk} - \sum_\ell\omega_\ell G_{\ell k}$ and distributing:
\begin{align*}
J_{ik} &= \frac{2}{\tau}\Bigl(\sum_j\omega_j G_{ji}G_{jk} - \underbrace{(\bc_{\text{soft}})_k}_{\sum_\ell\omega_\ell G_{\ell k}} \sum_j\omega_j G_{ji}\Bigr) \\
&= \frac{2}{\tau}\Bigl(\E_\omega[G_{\cdot i}G_{\cdot k}]
- \E_\omega[G_{\cdot i}]\,\E_\omega[G_{\cdot k}]\Bigr) = \frac{2}{\tau}\mathrm{Cov}_\omega(G_{\cdot i}, G_{\cdot k}).
\end{align*}

\textbf{(b) Input-independence.} Both $\bG$ (a fixed codebook) and $\boldsymbol{\omega}(\bc)$ (a function of $\bc$ alone) are independent of $(a,b)$, so $\mathbf{J}(\bc)$ is input-independent.

\textbf{(c) PSD.} For any $\mathbf{v}\in\R^4$, $\mathbf{v}^\top \mathbf{J}\mathbf{v} = (2/\tau)\Var_\omega(\sum_i v_i G_{\cdot i}) \ge 0$. Symmetry $J_{ik} = J_{ki}$ is immediate.

\textbf{(d) Positive diagonal.} $J_{kk} = (2/\tau)\Var_\omega(G_{\cdot k}) > 0$ iff $G_{\cdot k}$ is not $\boldsymbol{\omega}$-almost-surely constant. For each coordinate $k\in\{0,1,2,3\}$ on the 16-Boolean codebook, $G_{\cdot k}$ takes at least two distinct values, and softmax at finite $\tau{>}0$ assigns $\omega_j > 0$ for every $j$. Therefore $J_{kk} > 0$ throughout training.
\hfill$\square$

\paragraph{Gradient to $c_{ab}$ on every sample.}
The parameter gradient through the soft coefficient is $\partial\mathcal{L}/\partial c_k = (\partial\mathcal{L}/\partial z)\cdot \sum_i \psi_i(a,b)\, J_{ik}$. For $k=3$: \[ \frac{\partial \mathcal{L}}{\partial c_{ab}} = \frac{\partial \mathcal{L}}{\partial z}\cdot\bigl(1\cdot J_{03} + a\cdot J_{13} + b\cdot J_{23} + ab\cdot J_{33}\bigr). \] The leading term involves the constant basis element $\psi_0 = 1$, which is non-zero on every input. Since $J_{03}$ is input-independent and generically non-zero, $c_{ab}$ receives a non-zero gradient on \emph{every} training sample, in contrast to Multilinear-STE where $\partial\mathcal{L}/\partial c_{ab} = \delta\cdot ab$ vanishes on 75\% of Boolean inputs.

\subsection{Basis Optimality (Proposition~\ref{prop:basis-futility})}
\label{app:basis-futility}

\paragraph{Proposition~\ref{prop:basis-futility} (No affine basis modification improves STE)}

\emph{Goal.} An ideal STE basis for the interaction coefficient should satisfy three properties simultaneously: \textbf{(a)}~full \emph{coverage}: the gradient fires on every training sample ($\rho{=}1$); \textbf{(b)}~\emph{coherence}: the gradient consistently points in the correct direction ($\kappa{=}1$); \textbf{(c)}~zero \emph{STE bias}: the basis is equivalent to the canonical one up to positive rescaling ($B{=}0$). We show no affine product basis achieves all three.

\emph{Statement.} Let $\phi(x) = \alpha + \beta x$ with $\beta\neq 0$, and let $\tilde{\bpsi}(a,b) = (1,\ \phi(a),\ \phi(b),\ \phi(a)\phi(b))^\top$. Under uniform $(a,b)\in\{0,1\}^2$, define the following three properties of the interaction component $\tilde\psi_3(a,b) = \phi(a)\phi(b)$ (the 4th entry of $\tilde{\bpsi}$, corresponding to $c_{ab}$): \emph{coverage} $\rho := \Pr[\tilde\psi_3(a,b)\neq 0]$; \emph{coherence} $\kappa := |\E[\tilde\psi_3]|/\E[|\tilde\psi_3|] \in [0,1]$; \emph{STE bias} $B := \inf_{D} \E[\|\tilde{\bpsi} - D\bpsi_{\mathrm{can}}\|^2]$ where $D$ ranges over positive diagonal matrices. Then: (i) If $\phi(0) = 0$ or $\phi(1) = 0$: $\rho = 1/4$, $\kappa = 1$, $B = 0$. (ii) If $\phi(0)\neq 0$ and $\phi(1)\neq 0$: $\rho = 1$, but either $\kappa < 1$ or $B > 0$ (or both). (iii) No affine product basis simultaneously attains $\rho = 1$, $\kappa = 1$, and $B = 0$.

\emph{Proof.} Under STE with Boolean inputs $(a,b)\in\{0,1\}^2$, the gradient $\partial\mathcal{L}/\partial c_k$ is scaled by $\tilde\psi_k(a,b)$, so the properties $\rho$, $\kappa$, $B$ depend only on the four corner values of $\tilde\psi_3$. Since $\phi(x) = \alpha + \beta x$, define $p := \phi(0) = \alpha$ and $q := \phi(1) = \alpha + \beta$ (so $p \neq q$ since $\beta \neq 0$). The interaction basis function is $\tilde\psi_3(a,b) = \phi(a)\phi(b)$, which at the four corners evaluates to: $\tilde\psi_3(0,0) = \phi(0)\phi(0) = p^2$,\; $\tilde\psi_3(0,1) = \phi(0)\phi(1) = pq$,\; $\tilde\psi_3(1,0) = \phi(1)\phi(0) = pq$,\; $\tilde\psi_3(1,1) = \phi(1)\phi(1) = q^2$.

\textbf{Part (i): $p = 0$ (i.e.\ $\alpha=0$, $\phi(x)=\beta x$, $\beta>0$).} Substituting $p=0$: the corner values become $(0, 0, 0, q^2) = (0,0,0,\beta^2)$. Three of four corners are zero, so $\rho = \Pr[\tilde\psi_3 \neq 0] = 1/4$. The one non-zero value is $\beta^2 > 0$, so $|\E[\tilde\psi_3]| = \beta^2/4$ and $\E[|\tilde\psi_3|] = \beta^2/4$, giving $\kappa = 1$. The full basis evaluates to $\tilde{\bpsi}(a,b) = (1,\;\phi(a),\;\phi(b),\;\phi(a)\phi(b)) = (1,\;\beta a,\;\beta b,\;\beta^2 ab) = D\bpsi_{\mathrm{can}}$ with $D=\mathrm{diag}(1,\beta,\beta,\beta^2)$, a positive diagonal matrix (since $\beta>0$). Hence $B = 0$.

\textbf{Part (i'): $q = 0$ (i.e.\ $\phi(1)=0$, so $\alpha + \beta = 0$, $\phi(x)=\alpha(1-x)$).} The corner values become $(p^2, 0, 0, 0) = (\alpha^2, 0, 0, 0)$, so $\rho = 1/4$ and $\kappa = 1$ (same reasoning as Part~(i)). However, $\tilde\psi_1 = \phi(a) = \alpha - \alpha a = \alpha(1-a)$. For $B=0$ we need $\tilde\psi_1 = d_1 \cdot a$ for some $d_1 > 0$, but $\alpha(1-a)$ has value $\alpha$ at $a=0$ and $0$ at $a=1$, which is not proportional to $a = (0, 1)$. Hence $B > 0$.

\textbf{Part (ii): $p\neq 0$ and $q\neq 0$.} All four corner values $(p^2, pq, pq, q^2)$ are non-zero (since neither $p$ nor $q$ is zero), so $\rho = \Pr[\tilde\psi_3 \neq 0] = 1$.

\emph{Sub-case (ii-a): $pq > 0$ (same sign).} All four values $p^2, pq, pq, q^2 > 0$, so $|\E[\tilde\psi_3]| = \E[|\tilde\psi_3|]$ and $\kappa = 1$. However, expanding $\phi(a)\phi(b) = (\alpha+\beta a)(\alpha+\beta b) = \alpha^2 + \alpha\beta(a+b) + \beta^2 ab = p^2 + p\beta(a+b) + \beta^2 ab$, the interaction component $\tilde\psi_3$ contains a constant term $p^2$ and linear terms $p\beta \cdot a$, $p\beta \cdot b$. For $\tilde{\bpsi} = D\bpsi_{\mathrm{can}}$ we need $\tilde\psi_1 = d_1 \cdot a$; but $\tilde\psi_1 = \phi(a) = p + \beta a$ evaluates to $p$ at $a=0$ and $p+\beta$ at $a=1$. Since $p\neq 0$, this is not proportional to $a = (0, 1)$ for any $d_1 > 0$. Hence $B > 0$.

\emph{Sub-case (ii-b): $pq < 0$ (opposite signs).} The four corner values $(p^2, pq, pq, q^2)$ have mixed signs: $p^2, q^2 > 0$ but $pq < 0$. Under uniform $(a,b) \in \{0,1\}^2$:
\begin{align*}
\E[\tilde\psi_3] &= \tfrac{1}{4}(p^2 + pq + pq + q^2) = \tfrac{1}{4}(p^2 + 2pq + q^2) = \tfrac{(p+q)^2}{4}, \\ \E[|\tilde\psi_3|] &= \tfrac{1}{4}(p^2 + |pq| + |pq| + q^2) = \tfrac{1}{4}(p^2 - 2pq + q^2) = \tfrac{(p-q)^2}{4} = \tfrac{\beta^2}{4},
\end{align*}
where the last step uses $p - q = \alpha - (\alpha+\beta) = -\beta$. Therefore $\kappa = |\E[\tilde\psi_3]|/\E[|\tilde\psi_3|] = (p+q)^2/(p-q)^2 = (2\alpha+\beta)^2/\beta^2$. The Walsh/Hadamard basis maps $\{0,1\}$ to symmetric $\pm$ values: $\phi(0) = -\phi(1)$, i.e.\ $p = -q$. In our parameterization this requires $\alpha = -(\alpha+\beta)$, giving $\alpha = -\beta/2$. At this choice, $p + q = 0$, so $\kappa = 0$: exact cancellation. Since $p\neq 0$ (because $pq<0$), the same diagonal-rescaling argument as (ii-a) gives $B > 0$: $\tilde\psi_1 = p + \beta a$ cannot equal $d_1 a$ for any $d_1 > 0$.

\textbf{Part (iii).} Combining all cases: $B=0$ requires $p=0$ (Part~(i)), which gives $\rho = 1/4$. Every basis with $\rho = 1$ (Parts~(ii-a) and (ii-b)) has $B > 0$. Therefore no affine product basis achieves $(\rho, \kappa, B) = (1, 1, 0)$ simultaneously.
\hfill$\square$

\paragraph{Empirical confirmation.}
The Walsh basis ($\alpha = -\beta/2$, sub-case (ii-b) with $\kappa = 0$) achieves 86.86\% vs canonical 96.65\%; smoothed $\epsilon = 0.2$ (sub-case (ii-a), $B > 0$) achieves 76.95\%; expected-value and uniform bases (degenerate cases with $\beta = 0$) achieve 13.30\% and 12.42\%. See Appendix~\ref{app:ste-failures}.

\paragraph{Interpretation.}
The 25\% interaction-coverage ``starvation'' of Multilinear-STE is a fundamental property of every affine product basis, not a deficiency fixable by reparameterization. The correct response is to bypass basis-dependent gradients via Multilinear-CovJac (\S\ref{sec:basis-futility}).

\paragraph{Scope.}
Proposition~\ref{prop:basis-futility} restricts to affine product bases. Non-affine and non-product bases are outside our scope; we flag them as an open question.


\section{Background: Fuzzy Logic, WMC, and SPN View}
\label{app:background}

\subsection{Degree as logical complexity}
\label{app:degree}

The table below relates each monomial degree to its logical meaning. For $n = 2$, maximum degree 2 forces exactly 4 terms and exactly 4 parameters; the 4-parameter design of the Multilinear layer is \emph{determined} by the MLE structure, not chosen (Table~\ref{tab:degree-complexity}).

\begin{table}[t]
\centering
\caption{Monomial degree as logical complexity for $n=2$ inputs.}
\label{tab:degree-complexity}
\begin{tabular}{cccl}
\toprule
Degree & Term & Logical meaning & Example gates \\
\midrule
0 & $c_0$       & Unconditional output & FALSE, TRUE \\ 1 & $c_a a$     & Independent input-$a$ effect & A, $\neg$A \\ 1 & $c_b b$     & Independent input-$b$ effect & B, $\neg$B \\ 2 & $c_{ab} ab$ & Pairwise interaction & AND, XOR, NAND \\
\bottomrule
\end{tabular}
\end{table}

\subsection{Fuzzy logic and the product t-norm}
\label{app:fuzzy}

Under the product t-norm~\citep{klement2000triangular} (probabilistic fuzzy logic), fuzzy connective evaluation is \emph{identical} to MLE evaluation: \[ a \wedge b = ab,\qquad a \vee b = a + b - ab,\qquad \neg a = 1 - a,\qquad a \oplus b = a + b - 2ab. \] These are not approximations. The MLE \emph{is} the fuzzy truth-degree function under the product t-norm~\citep{vankrieken2022analyzing}. \L{}ukasiewicz ($\min$) and Gödel ($\max$) fuzzy logics yield piecewise-linear extensions instead.

\subsection{Probabilistic circuit and WMC interpretation}
\label{app:wmc}

If inputs are independent, $X\sim\mathrm{Ber}(a)$ and $Y\sim\mathrm{Ber}(b)$, then \[ F(a,b) = \Pr[g(X,Y)=1 \mid X\sim\mathrm{Ber}(a),\,Y\sim\mathrm{Ber}(b)]. \] This is the \emph{Weighted Model Count} (WMC)~\citep{deraedt2007problog} of gate $g$ under independent Bernoulli atoms, the same computation used in ProbLog and probabilistic circuits~\citep{choi2020probabilistic}. For a single 2-input gate, WMC and MLE evaluation are the same operation.

\subsection{Corner basis as a sum-product network}
\label{app:corner-spn}

The corner basis functions $\phi_{ij}(a,b) = a^i(1-a)^{1-i}\,b^j(1-b)^{1-j}$ used by
IWP~\citep{ruttgers2026narrowing} are exactly the joint probability
mass function of two independent Bernoullis. Evaluating the IWP forward pass $z = \sum_{ij} s_{ij}\,\phi_{ij}(a,b)$ is equivalent to running a shallow sum-product network with 4 product nodes (one per corner) and a single sum node. The multilinear basis $\{1,a,b,ab\}$ and the corner basis span the same space and are related by the invertible linear map of \S\ref{sec:method} / Appendix~\ref{app:cost}.

\subsection{Unifying principle}
\label{app:unifying}

The three views (symbolic Boolean logic, probabilistic WMC, and differentiable Multilinear learning) compute the same quantity under different conventions:
\begin{align*}
&\{0,1\}\text{ inputs (Boolean logic)}
\;\longleftrightarrow\; [0,1]\text{ inputs (WMC / Bayes net)} \\
&\qquad\longleftrightarrow\; \text{gradient flows (Multilinear layer)}.
\end{align*}
The Multilinear layer (instantiated as Multilinear-STE or Multilinear-CovJac; \S\ref{sec:method}) makes this explicit: it directly parameterizes the MLE coefficients and snaps to the nearest Boolean gate, so gradient-based training and Boolean inference are unified within a single algebraic object.


\section{Computational Cost and Corner-Polynomial Bijection}
\label{app:cost}

\subsection{Method comparison}
\label{app:method-table}

Table~\ref{tab:method-summary} summarizes all methods compared in this paper, including parameter counts, forward/backward semantics, and discretization-gap properties.

\begin{table}[t]
\centering
\small
\caption{Summary of all methods compared. Params = learnable parameters per
neuron. DG = discretization gap (structural, not empirical).}
\label{tab:method-summary}
\begin{tabular}{lccccc}
\toprule
Method & Params & Train fwd & Eval fwd & Backward & DG \\
\midrule
Soft-Mix~\citep{petersen2022deep} & 16 & soft mixture & argmax & autograd & $>0$ \\
Gumbel-ST~\citep{jang2017categorical} & 16 & hard (Gumbel) & argmax & Gumbel-ST & small \\
IWP~\citep{ruttgers2026narrowing} & 4 & bilinear & bilinear & exact & $>0$ \\
IWP-STE (ablation) & 4 & snap & snap & corner STE & $0$ \\
\textbf{Multilinear-STE} (ours) & \textbf{4} & snap & snap & poly STE & $\mathbf{0}$ \\ \textbf{Multilinear-CovJac} (ours) & \textbf{4} & soft-VQ & snap & CovJac & $\mathbf{0}$ (emp.) \\
\bottomrule
\end{tabular}
\end{table}

\subsection{Per-neuron cost summary}
\label{app:cost-table}

Table~\ref{tab:cost} compares the per-neuron computational cost of all four methods. At deployment, all methods evaluate a single hard Boolean gate (${\sim}7$ ops: compute $ab$, multiply by 4 coefficients, sum). The training cost differs: Soft-Mix and Gumbel-ST evaluate all 16 gate outputs and compute a softmax mixture (174 ops); M-STE evaluates only the snapped polynomial (7 ops, the same as deployment);
M-CovJac computes the soft-VQ proximity weights over 16 codebook
entries (174 ops, similar to Soft-Mix). The $4\times$ parameter reduction (16$\to$4 per neuron) reduces the Adam optimizer state from 48 to 12 floats per neuron (4 parameters $\times$ 3: value + first-moment + second-moment).

\begin{table}[t]
\centering
\footnotesize
\caption{Per-neuron computational cost.
All methods deploy the same hard Boolean gate.}
\label{tab:cost}
\begin{tabular}{@{}lcccc@{}}
\toprule
Method & Params/neuron & Train (ops/sample) & Deploy (ops/sample) & Memory (floats) \\
\midrule
Soft-Mix       & 16 & ${\sim}$175 & 7 & 48 \\
Gumbel-ST      & 16 & ${\sim}$175 & 7 & 48 \\
M-STE (ours)   & 4  & 7   & 7 & 12 \\
M-CovJac (ours)& 4  & 7 & 7 & 12 \\
\bottomrule
\end{tabular}

{\tiny All methods have small batch-independent overhead ($<$0.4 ops/sample
at $B{=}512$; see \S\ref{app:opcounts} summary).}
\end{table}

\subsection{Computational cost: full op counts}
\label{app:opcounts}

\paragraph{Per-sample forward (Soft-Mix).}
Soft-Mix evaluates all 16 gate polynomials per sample:
\begin{itemize}
\item Softmax over 16 logits: 16 exp $+$ 15 add $+$ 16 div $= 47$ ops.
\item Compute $ab$: 1 multiply (shared across all 16 gates).
\item 16 gate evaluations $g_j(a,b) = c_{0,j} + c_{a,j}a + c_{b,j}b
+ c_{ab,j}ab$: each needs 3 multiply ($c_{a,j}a$, $c_{b,j}b$, $c_{ab,j}\cdot ab$) $+$ 3 add $= 6$ ops; $16 \times 6 + 1 = 97$ ops.
\item Weighted sum $z = \sum_j \pi_j g_j$: 16 multiply $+$ 15 add
$= 31$ ops.
\item \textbf{Total: $47 + 97 + 31 = 175$ ops/sample/neuron.}
\end{itemize}
Backward: $\partial\mathcal{L}/\partial\pi_j = \delta\cdot g_j(a,b)$ for each $j$, plus softmax Jacobian: $\approx 100$ ops/neuron.

\paragraph{Per-sample forward (Gumbel-ST).}
Same 16-gate evaluation as Soft-Mix, plus Gumbel noise:
\begin{itemize}
\item Sample 16 Gumbel variables ($-\log(-\log u)$): $16\times 3 = 48$ ops.
\item Add to logits and take argmax: $16 + 15 = 31$ ops.
\item Evaluate the selected gate polynomial: $7$ ops.
\item \textbf{Total: $48 + 31 + 7 = 86$ ops (hard forward);
${\sim}174$ ops with soft backward (Gumbel-softmax).}
\end{itemize}
The hard forward is cheaper than Soft-Mix (one gate, not 16), but the Gumbel-softmax backward still requires all 16 gate values, bringing the effective training cost to ${\sim}174$ ops.

\paragraph{Per-sample forward (Multilinear-STE).}
Evaluates \emph{one} committed polynomial per sample:
\begin{itemize}
\item $z = c_0 + c_a a + c_b b + c_{ab} ab$: 1 multiply ($ab$) $+$
3 multiply ($c_k \cdot \text{basis}_k$) $+$ 3 add $= 7$ ops/sample/neuron.
\item Gate selection (snap) is computed once per layer, amortized
across the batch (\S\ref{app:opcounts} below).
\end{itemize}
Backward STE: $\nabla_{\bc}\mathcal{L} = \delta\cdot[1,a,b,ab]^\top = [\delta,\;\delta a,\;\delta b,\;\delta\!\cdot\!ab]$: 3 multiply ($\delta$ is free for $c_0$) $= 3$ ops/neuron.

\paragraph{Per-sample forward (Multilinear-CovJac).}
The soft-VQ weights $\omega_j$ and the soft coefficient $\bc_{\text{soft}} = \sum_j\omega_j\bG_j$ depend only on $\bc$ and the fixed codebook $\bG$, not on the input $(a,b)$. They are therefore computed \textbf{once per neuron per batch}: distances $\|\bc-\bG_j\|^2$ for 16 gates (${\sim}120$ ops, using $\|\bc\|^2 - 2\bc^\top\bG_j + \|\bG_j\|^2$ with $\|\bG_j\|^2$ precomputed), softmax ($47$ ops), weighted sum (${\sim}30$ ops); total batch overhead ${\sim}200$ ops/neuron. Per sample, each neuron evaluates $z = \bc_{\text{soft}}^\top\bpsi(a,b) = {\sim}7$ ops (same as STE). Amortized over batch $B{=}512$: $200/512 \approx 0.4$ extra ops/sample, giving an effective per-sample cost of ${\sim}7.4$ ops. The CovJac backward adds a $4{\times}4$ Jacobian--vector product (${\sim}28$ ops/neuron, also batch-independent). At deployment, CovJac uses the same hard gate as STE (${\sim}7$ ops).

\paragraph{Summary.}
Per sample at training: STE ${\sim}7$ ops; CovJac ${\sim}7$ ops;
Soft-Mix ${\sim}175$ ops. STE and CovJac are both ${\sim}25\times$
cheaper than Soft-Mix per sample. All methods also incur a small batch-independent overhead per neuron:
M-STE ${\sim}120$ ops (snap distances $\|\bc-\bG_j\|^2$),
M-CovJac ${\sim}200$ ops (distances $+$ softmax $+$ weighted sum),
Soft-Mix ${\sim}47$ ops (softmax on logits).
At $B{=}512$ these amortize to $<$0.4 ops/sample for all methods. At deployment all methods evaluate a single hard gate: ${\sim}7$ ops.

\paragraph{Layer-level $L_2$ snap.}
The $L_2$ snap is computed \emph{once per layer} and shared across the batch, so its $O(N\cdot 64)$ cost is batch-independent (approximately 5\% of total at $B{=}512$). Using the identity \[ \|\bc-\bG_j\|^2 = \|\bc\|^2 - 2\bc\bG^\top + \|\bG_j\|^2, \] the snap reduces to a single cuBLAS matmul.

\paragraph{Optimizer state.}
Adam stores per parameter: value $+$ first-moment $+$ second-moment $= 3$ floats. Multilinear: $4 \times 3 = 12$ floats/neuron.
Soft-Mix: $16 \times 3 = 48$ floats/neuron.
\textbf{$4\times$ reduction in optimizer memory.}

\subsection{Corner--polynomial bijection (explicit matrices)}
\label{app:bijection}

\paragraph{Proposition (Bijection).}
The multilinear parameterization $\bc=[c_0,c_a,c_b,c_{ab}]$ and the corner parameterization $\bs=[s_{00},s_{10},s_{01},s_{11}]$ are related by the invertible linear map $\bc=M\bs$ where \[ M = \begin{pmatrix} 1 & 0 & 0 & 0 \\ -1 & 1 & 0 & 0 \\ -1 & 0 & 1 & 0 \\ 1 & -1 & -1 & 1 \end{pmatrix},\qquad M^{-1} = \begin{pmatrix} 1 & 0 & 0 & 0 \\ 1 & 1 & 0 & 0 \\ 1 & 0 & 1 & 0 \\ 1 & 1 & 1 & 1 \end{pmatrix}. \] Both parameterizations span the same function space and represent the same 16 Boolean gates. Consequences: (a) the 17pp IWP-STE vs Multilinear-STE accuracy gap at $k{=}64\text{k}$ (\S\ref{sec:ablation}) cannot arise from function class or expressive capacity; it must arise from optimization dynamics, specifically the STE gradient structure (verified by the $2{\times}2$ experiment in Appendix~\ref{app:basis-constraint} and by Lemma~\ref{lem:coverage} in Appendix~\ref{app:coverage}).

\subsection{Voronoi geometry of the $L_2$ snap}
\label{app:voronoi}

The Voronoi cells of the 16 gate vectors are convex polyhedra in $\R^4$; nearest-gate distances between pairs of gates range from 1 (adjacent gates, e.g.\ AND to $a\wedge\neg b$) to $\sqrt{11}$ (XOR $(0,1,1,-2)$ to AND $(0,0,0,1)$). No geometric uniformity is required: what matters is that every boundary of the continuous $\bc$ space lies in the interior of some Voronoi cell, so the snap is well-defined almost everywhere.


\section{STE Basis Failure Catalogue}
\label{app:ste-failures}

\subsection{Alternative STE bases tested}
\label{app:ste-bases}

All experiments below use the Multilinear-STE forward pass ($L_2$ snap to the nearest Boolean gate) and swap only the backward basis $\tilde{\bpsi}(a,b)$ used in the STE gradient $\delta\cdot\tilde{\bpsi}(a,b)$. Results at $k{=}64\text{k}$, $L{=}6$, 3 seeds, last-10 hard accuracy (Table~\ref{tab:ste-basis-catalogue}).

\begin{table}[t]
\centering
\caption{STE basis failure catalogue on MNIST ($k=64\text{k}$, $L=6$, 3 seeds unless noted). \textbf{Bold} = reference (canonical basis).}
\label{tab:ste-basis-catalogue}
\scriptsize
\begin{tabular}{@{}p{2.8cm}p{3.5cm}cp{4.5cm}@{}}
\toprule
Basis & Description & Acc (\%) & Interpretation \\
\midrule
Canonical $(1,a,b,ab)$ & Multilinear-STE (reference) & \textbf{98.11 $\pm$ 0.03} & $(\rho,\kappa,B) = (1/4,1,0)$, Part (i) \\ Walsh $\{-1,+1\}$: $\phi(x){=}2x{-}1$ & $\alpha{=}-1,\beta{=}2$, centered & 94.92 $\pm$ 0.01 & Sub-case (ii-b), $p{=}-q$, $\kappa{=}0$, coherence cancellation \\ Smoothed $\epsilon{=}0.2$: $(1,a{+}\epsilon,b{+}\epsilon,(a{+}\epsilon)(b{+}\epsilon))$ & $\phi(x){=}x{+}0.2$ & 10.14 $\pm$ 2.63 & Sub-case (ii-a), $p,q{>}0$, $B{>}0$, collapses at $k{=}64\text{k}$ \\ Expected-value $(1, 0.5, 0.5, 0.25)$ & $\beta{=}0$; basis is constant & 13.30$^*$ & Degenerate, $\beta{=}0$ excluded in Prop.~\ref{prop:basis-futility} \\ Uniform $(1,1,1,1)$ & $\alpha{=}1,\beta{=}0$ & 12.42$^*$ & Degenerate \\
\bottomrule
\end{tabular}
\end{table}

$^*$Single seed, $k{=}8\text{k}$ (degenerate bases collapse to near-random regardless of scale; re-run at $k{=}64\text{k}$ not informative).

\subsection{Walsh cancellation is the cleanest predicted failure}
\label{app:walsh}

Setting $\alpha{=}-\beta/2$ (Walsh basis) gives $p{=}-q$, so the corner values of $\tilde\psi_3{=}\phi(a)\phi(b)$ are $(p^2, -p^2, -p^2, p^2)$. The expectation under uniform Boolean $(a,b)$ is $\E[\tilde\psi_3] = (1/4)(p^2 - p^2 - p^2 + p^2) = 0$, producing \emph{exact} coherence cancellation $\kappa = 0$. Proposition~\ref{prop:basis-futility} predicts this basis cannot deliver a coherent $c_{ab}$ update; the 94.92\% accuracy (3.2pp below canonical at $k{=}64\text{k}$) is the empirical consequence.

\subsection{Smoothed basis fails for a different reason}
\label{app:smoothed}

With $\phi(x){=}x+\epsilon$, $\epsilon{=}0.2$, we have $p{=}0.2,q{=}1.2{>}0$, so $\kappa{=}1$ but the basis components differ from canonical: specifically, $\tilde\psi_3{=}(0.2{+}a)(0.2{+}b) = 0.04 + 0.2(a{+}b){+}ab$ contains a non-removable constant and linear term, making $B{>}0$ in the Proposition~\ref{prop:basis-futility} sense. The 10.14\% accuracy at $k{=}64\text{k}$ (near random chance on 10-class MNIST) is a catastrophic failure (far worse than Walsh), confirming that STE bias ($B{>}0$) is more destructive than coherence loss ($\kappa{=}0$).

\subsection{Degenerate $\beta=0$ bases}
\label{app:degenerate}

The expected-value and uniform bases have $\beta{=}0$, meaning $\tilde\psi_k$ no longer depends on $(a,b)$ in the linear coordinates. These are outside the scope of Proposition~\ref{prop:basis-futility} (which assumes $\beta\neq 0$), but we include them for completeness. Both collapse to random-guess-level accuracy (12--13\% on 10-class MNIST-bin), confirming that non-informative backward bases simply stop learning. They are instructive negative controls but not interesting counterexamples.


\section{Experimental Details and Additional Results}
\label{app:experiments-details}

This appendix contains: training protocol and hyperparameters (\S\ref{app:hyperparams}); binarization analysis and architecture baselines (\S\ref{app:binarization}); discretization gap (\S\ref{app:gap-decomp}); gate-entropy diagnostic (\S\ref{app:entropy}); width scaling (\S\ref{app:width-scaling}); CIFAR-10 full table (\S\ref{app:cifar-full}); scaling to larger datasets (\S\ref{app:scale-up}); gradient-norm diagnostic (\S\ref{app:gradnorm}); interaction-task scaling (\S\ref{app:h7}); depth scaling (\S\ref{app:depth-h6}); gate codebook distribution (\S\ref{app:codebook-neurons}); and temperature sensitivity (\S\ref{app:tau}).

\subsection{Hyperparameters and training protocol}
\label{app:hyperparams}

All methods are trained with Adam (lr~$=0.01$, $\beta_1 = 0.9$, $\beta_2 = 0.999$) for 50,000 iterations with batch size 512. Gate logits are initialized i.i.d.\ $\mathcal{N}(0, \sigma^2)$ with $\sigma = 1.0$ (the standard deviation of the initialising normal distribution for each 16-dim Soft-Mix logit vector and each 4-dim Multilinear coefficient vector). We report mean $\pm$ std over 3 independent seeds. Primary metric: \textbf{last-10 hard accuracy}, the mean hard-evaluation accuracy over the last 10 evaluation checkpoints, each spaced 1{,}000 iterations apart, which reduces noise from the final checkpoint. The best-checkpoint accuracy differs from last-10 by up to 1.1pp (M-CovJac on MONK's-2). Method rankings are preserved on most datasets; on Splice, the M-STE/M-CovJac ordering reverses under best-checkpoint, confirming last-10 as the more stable ranking metric.

\paragraph{Statistical significance.}
\label{app:significance}
With $n{=}3$ seeds, paired $t$-tests have only 2 degrees of freedom (critical $t{=}4.30$ for $p{<}0.05$), limiting the power of per-dataset significance tests. On CIFAR-10, the CovJac-vs-STE gap ($+2.95$pp) is significant ($t{=}8.70$, $p{=}0.013$); other individual comparisons do not reach $p{<}0.05$ at $n{=}3$, including the large MONK's-2 gap ($+7.53$pp, $p{=}0.052$) due to high variance. The stronger evidence is the \emph{consistency} of CovJac's advantage across all seven datasets: the probability that a method wins on all 7 datasets by chance (under a null of equal performance) is $2^{-7} < 0.01$, a non-parametric sign test at $p{<}0.01$.

Evaluation is performed every 1{,}000 iterations on a held-out test split. Binary MNIST is threshold-binarized at 0.5; CIFAR-10 binary uses the cifar-10-31-thresholds protocol (31 per-channel thresholds producing a 95{,}232-dim binary input).

\paragraph{Wiring.}
Each neuron receives exactly two distinct inputs from the previous layer. We use a stride-based locality-preserving scheme: the term ``stride'' refers to the pixel-wise shift when sampling input pairs, analogous to convolutional stride in CNNs. \textbf{Stride\,=\,2} (non-overlapping): pairs $(x_0,x_1),(x_2,x_3),\ldots$, yielding $d/2$ pairs. \textbf{Stride\,=\,1} (overlapping): pairs $(x_0,x_1),(x_1,x_2),\ldots$, yielding $d{-}1$ pairs. If $m$ gates require more pairs than adjacent elements provide, we extend to larger gaps while preserving locality. For gap $g{=}1$, we first use even-indexed pairs $(x_0,x_1),(x_2,x_3),\ldots$, then odd-indexed pairs $(x_1,x_2),(x_3,x_4),\ldots$. When these are exhausted, we consider pairs with gap $g{=}2$: $(x_0,x_2),(x_2,x_4),\ldots$, then $(x_1,x_3),(x_3,x_5),\ldots$, and so on for increasing $g$. The wiring is fixed across seeds within a configuration but re-sampled across configurations.

\paragraph{Architecture.}
$L$ hidden layers of width $k$ followed by a GroupSum readout (10 equal groups summed within, producing 10 class logits). Softmax cross-entropy loss. Primary $(k,L)$ is $(64\text{k},6)$ on MNIST and $(128\text{k},6)$ on CIFAR-10. For the depth sweep (\S\ref{sec:depth}), $L\in\{3,6,9,12\}$ at $k{=}64\text{k}$ on MNIST binary.

\paragraph{Compute resources.}
All experiments were run on a single NVIDIA B200 GPU (192\,GB HBM3e) per run.

\paragraph{Datasets.}
MNIST~\citep{lecun1998mnist}: binarized at threshold 0.5, 784-dim, 10 classes. SVHN~\citep{netzer2011svhn} and CIFAR-10/100~\citep{krizhevsky2009cifar}: each channel binarized at 31 equally spaced thresholds, yielding 95{,}232-dim binary inputs. Adult and Splice~\citep{dua2019uci}: binary-encoded tabular features (116-dim and 4{,}095-dim respectively). MONK's-2~\citep{thrun1991monks}: 17-dim binary encoding of 6 categorical attributes; the target function requires all pairwise interactions.

\paragraph{Binarization information loss.}
\label{app:binarization}
Input binarization discards sub-threshold resolution. For CIFAR-10/100, 31 equally spaced thresholds per channel produce a thermometer code with 32 possible states per pixel-channel, encoding $\log_2 32 = 5$ effective bits out of the original 8, a 37.5\% information loss per pixel-channel. For MNIST (threshold at 0.5), each pixel is reduced from 8 bits to 1, a 87.5\% loss in principle, though MNIST pixel values are near-binary and the lost bits carry little signal. Table~\ref{tab:mlp-cnn-baseline} compares DLGNs to standard MLPs and CNNs on the \emph{same binarized inputs}, isolating the effect of the Boolean gate constraint from the binarization bottleneck.

\textbf{MLP architecture.} Fully connected: input $\to$ $L$ hidden layers of width $h$ with ReLU activations $\to$ linear output. Trained with Adam (lr$=0.001$), batch 512, 50k iterations. The flat binarized input (784-dim for MNIST, 95{,}232-dim for CIFAR-10) is fed directly.

\textbf{CNN architecture.} The binarized input is reshaped to a multi-channel image (MNIST: $1{\times}28{\times}28$; CIFAR-10: $93{\times}32{\times}32$, i.e.\ 3 RGB channels $\times$ 31 thresholds). $L$ convolutional layers (zero-padding, ReLU) with $2{\times}$ max-pooling every 2 layers, followed by global average pooling and a linear classifier. Channels: $32 \to 64 \to 128$ ($L{=}3$) or $32 \to 64 \to 128 \to 256 \to 256 \to 256$ ($L{=}6$). Kernel size: $3{\times}3$ or $2{\times}2$ as noted. Same optimizer and budget as MLP.

\begin{table}[t]
\centering
\small
\caption{Binarized-input baselines on the same binary inputs as DLGN.
Last-10 accuracy $\pm$ std, 3 seeds. DLGNs deploy as Boolean circuits with zero floating-point operations; the accuracy gap vs CNN/MLP reflects the Boolean gate constraint, not binarization.}
\label{tab:mlp-cnn-baseline}
\scriptsize
\begin{tabular}{llrrccc}
\toprule
Model & Config & Train params & Deploy params & Inference & MNIST & CIFAR-10 \\
\midrule
MLP & $L{=}3$, $h{=}64$   & 59k  & 59k float & 59k MACs & 96.95$\pm$0.04 & 12.28$\pm$1.68 \\
MLP & $L{=}6$, $h{=}128$  & 184k & 184k float & 184k MACs & 97.63$\pm$0.07 & 28.26$\pm$1.63 \\
CNN $3{\times}3$ & $L{=}3$ & 94k  & 94k float & ${\sim}$2M MACs & 98.86$\pm$0.06 & --- \\
CNN $2{\times}2$ & $L{=}6$ & 712k & 712k float & ${\sim}$10M MACs & 98.90$\pm$0.07 & 75.13$\pm$0.83 \\
CNN $3{\times}3$ & $L{=}6$ & 1.60M & 1.60M float & ${\sim}$30M MACs & 99.43$\pm$0.00 & 79.75$\pm$0.90 \\
\midrule
M-CovJac & $L{=}6$ & 1.54M$^*$ & \textbf{0 float}$^\dagger$ & \textbf{0 MACs}$^\dagger$ & 98.30$\pm$0.04 & 58.97$\pm$0.26 \\
\bottomrule
\end{tabular}

{\tiny $^*$MNIST $k{=}64\text{k}$ (1.54M); CIFAR-10 $k{=}128\text{k}$ (3.07M).
$^\dagger$Deployed as a combinational Boolean circuit with \textbf{zero float parameters} and zero MACs. The circuit stores gate type (4 bits/neuron; ${\approx}$192\,KB for MNIST $k{=}64\text{k}$, $L{=}6$) plus wiring indices. All training parameters are discarded after gate selection.}
\end{table}

\paragraph{Where do the accuracy gaps come from?}
The accuracy gap between a CNN and a DLGN on the same binarized input has two sources: (1)~the \emph{Boolean gate constraint}: each DLGN neuron outputs a Boolean function of two binary inputs, whereas CNN neurons compute continuous multiply-accumulate; and (2)~the \emph{readout bottleneck}: GroupSum partitions the $k$ gate outputs into $C$ equal groups and sums within each, a fixed non-learned projection. A natural question is whether adding a small number of floating-point operations could close the gap while preserving the Boolean hidden layers.

\textbf{Learned readout (replacing GroupSum).} Replacing GroupSum with a trainable linear layer $\R^k \to \R^C$ adds $k \times C$ MACs (e.g., 640k for MNIST at $k{=}64\text{k}$, $C{=}10$), negligible compared to the millions of MACs in a CNN. The Boolean hidden layers remain pure combinational logic; only the final readout uses arithmetic, which can be implemented via DSP slices or LUTs on FPGA. This is the lowest-cost intervention because GroupSum assigns each gate to exactly one class with unit weight, which is a highly restricted readout.

\textbf{Learned input embedding.} Replacing hard binarization with a learned linear projection would recover some of the 37.5\% information lost per pixel-channel, but would require multiply-accumulate hardware at the input.

\textbf{Comparison to quantized CNNs on FPGA.} Modern quantized CNNs (INT8, binary neural networks~\citep{hubara2016binarized}) also map efficiently to FPGAs via XNOR+popcount~\citep{rastegari2016xnor} or DSP slices. The unique advantage of DLGNs is not FPGA-compatibility per se, but the \emph{extreme} end of the efficiency spectrum: fully combinational evaluation (no clock, no weight storage, no multiply hardware), relevant for applications where even DSP blocks or BRAM are too expensive (IoT sensors, in-pixel processing). The accuracy gap is the cost of this constraint; this paper's contribution (CovJac) narrows the gap within the DLGN family.

Both readout and input-embedding directions are orthogonal to the gate-selection mechanism studied here and are left to future work.

\paragraph{Wiring sensitivity.}
\label{app:wiring-sensitivity}
This paper uses fixed random wiring (each neuron connects to 2 randomly chosen inputs from the previous layer). To quantify the impact of wiring topology, we test CovJac with 5 different random wiring seeds on CIFAR-10 ($k{=}128\text{k}$, $L{=}6$, Table~\ref{tab:wiring-ablation}). The wiring-induced variance is only $\pm 0.28$pp (range 0.86pp), while the gap to a CNN $2{\times}2$ on the same binarized input is $+16.2$pp (Table~\ref{tab:mlp-cnn-baseline}). This indicates that \textbf{wiring topology is not the bottleneck}: different random wirings perform similarly, and learnable wiring~\citep{bacellar2024dwn} is unlikely to close the CNN gap. The dominant bottleneck is the 2-input Boolean gate constraint itself: each DLGN neuron computes a discrete function of 2 binary inputs, while each CNN $2{\times}2$ neuron computes a continuous multiply-accumulate over 4 real-valued inputs. Closing this gap requires increasing the number of inputs per gate ($k$-input generalization) or relaxing the Boolean constraint at selected layers.

\begin{table}[t]
\centering
\small
\caption{Wiring sensitivity: CovJac with 5 different random wiring
seeds. Same gate-init seed, same training protocol. Last-10 accuracy.}
\label{tab:wiring-ablation}
\begin{tabular}{lcc}
\toprule
Dataset & Mean $\pm$ std & Range \\
\midrule
CIFAR-10 ($k{=}128\text{k}$) & 58.56$\pm$0.28 & 58.04--58.90 \\
\bottomrule
\end{tabular}
\end{table}

\paragraph{Best-checkpoint accuracy (Table~\ref{tab:best-acc}).}
Table~\ref{tab:mnist-main} reports last-10 accuracy as the primary metric. Table~\ref{tab:best-acc} reports the corresponding best-checkpoint accuracy (maximum over all evaluation checkpoints, averaged across 3 seeds). Method rankings are identical under both metrics on all datasets.

\begin{table}[t]
\centering
\footnotesize
\caption{Best-checkpoint accuracy $\pm$ std (3 seeds).
Same configurations as Table~\ref{tab:mnist-main}. \textbf{Bold} = best per column.}
\label{tab:best-acc}
\scriptsize
\begin{tabular}{@{}lccccccc@{}}
\toprule
Method & Adult & Splice & MONK's-2 & MNIST & SVHN & CIFAR-10 & CIFAR-100 \\
\midrule
Soft-Mix       & 85.01{\tiny$\pm$.03} & 97.40{\tiny$\pm$.09} & 82.20{\tiny$\pm$.96} & 98.32{\tiny$\pm$.03} & 68.37{\tiny$\pm$.07} & 58.32{\tiny$\pm$.05} & 28.21{\tiny$\pm$.38} \\
Gumbel-ST      & 85.01{\tiny$\pm$.08} & 97.20{\tiny$\pm$.09} & 82.46{\tiny$\pm$1.07} & 98.21{\tiny$\pm$.08} & 67.06{\tiny$\pm$.28} & 58.08{\tiny$\pm$.19} & 28.08{\tiny$\pm$.33} \\
M-STE (ours)   & 85.11{\tiny$\pm$.04} & 97.72{\tiny$\pm$.09} & 79.34{\tiny$\pm$1.05} & 98.15{\tiny$\pm$.02} & 67.01{\tiny$\pm$.18} & 56.95{\tiny$\pm$.06} & 25.02{\tiny$\pm$.16} \\
M-CovJac (ours)& \textbf{85.13}{\tiny$\pm$.03}$^\dagger$ & \textbf{97.53}{\tiny$\pm$.24} & \textbf{87.15}{\tiny$\pm$1.92} & \textbf{98.36}{\tiny$\pm$.03} & \textbf{69.19}{\tiny$\pm$.18} & \textbf{59.18}{\tiny$\pm$.26} & \textbf{28.91}{\tiny$\pm$.14} \\
\bottomrule
\end{tabular}

{\tiny $^\dagger$Adult: $\tau{=}0.1$.}
\end{table}

\begin{figure}[t]
\centering
\includegraphics[width=0.5\linewidth]{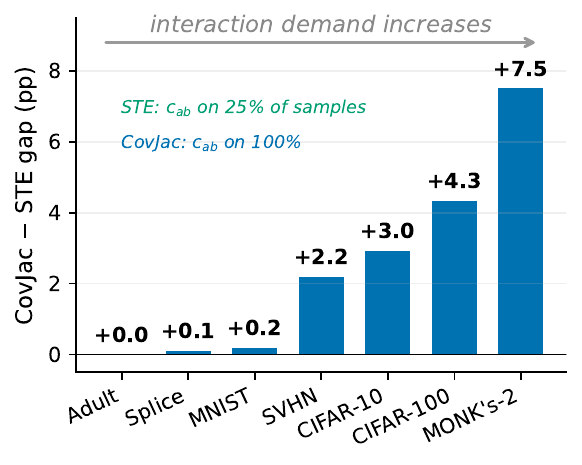}
\caption{CovJac$-$STE gap grows monotonically with interaction
demand across 7 datasets. $c_{ab}$ starvation (25\% coverage) hurts more as tasks require more interactive gates.}
\label{fig:interaction-scaling}
\end{figure}

\subsection{Discretization gap}
\label{app:gap-decomp}

The \emph{discretization gap} (DG) measures the accuracy difference between the training-time forward pass and the deployed hard Boolean circuit: $\mathrm{DG} = A^{\text{train}} - A^{\text{hard}}$. For Multilinear-STE, $\mathrm{DG}{=}0$ by construction (both phases use the quantized gate). For Multilinear-CovJac, the soft-VQ training forward may differ from the hard deployment gate; empirically $\mathrm{DG}\,{\leq}\,0.2\%$ on all main datasets (${\leq}\,1\%$ on MONK's-2, which has the smallest $k{=}136$). The generalization gap is $\mathrm{Gen} = A^{\text{hard, train}} - A^{\text{hard, test}}$.

\subsection{Gate-entropy diagnostic (Observation~1)}
\label{app:entropy}

For each training run we log the per-layer gate-distribution entropy $H(\bpi^{(l)}) = -\sum_j \pi^{(l)}_j \log \pi^{(l)}_j$ at every evaluation checkpoint. For Soft-Mix, $\bpi^{(l)}$ is the softmax over the 16 gate logits at layer $l$; for Multilinear methods, $\bpi^{(l)}$ is derived from the $L_2$ distances to the 16 codebook entries.

\paragraph{Expected pattern (Observation~1).}
In Soft-Mix runs that successfully train at moderate depth, entropy should drop first in the output-adjacent layers and only later in the input-adjacent layers. In stalling runs ($L{=}12$), deep layers should retain near-maximum entropy $\log 16 \approx 2.77$ nats throughout training, consistent with greedy output-to-input commitment failing to reach them.

\subsection{Width-scaling tables}
\label{app:width-scaling}

Width sweep on MNIST binary: $k\in\{8\text{k},16\text{k},32\text{k},64\text{k}\}$, $L\in\{1,3,6\}$, 3 seeds each, 50k iterations (Table~\ref{tab:width-scaling}).

\begin{table}[t]
\centering
\caption{Width scaling on MNIST ($L=1,3,6$, 3 seeds). Last-10 accuracy. \textbf{Bold} = best per row.}
\label{tab:width-scaling}
\small
\begin{tabular}{lcccc}
\toprule
Width & Soft-Mix & Gumbel-ST & M-STE (ours) & M-CovJac (ours) \\
\midrule
\multicolumn{5}{c}{\textit{$L=1$}} \\
\midrule
$k{=}8\text{k}$  & 91.89 $\pm$ 0.03 & 91.73 $\pm$ 0.03 & 91.39 $\pm$ 0.10 & 91.62 $\pm$ 0.13 \\ $k{=}16\text{k}$ & 93.22 $\pm$ 0.10 & 93.15 $\pm$ 0.13 & 93.25 $\pm$ 0.05 & \textbf{93.48 $\pm$ 0.02} \\ $k{=}32\text{k}$ & 94.56 $\pm$ 0.15 & 94.72 $\pm$ 0.08 & 95.10 $\pm$ 0.08 & \textbf{95.14 $\pm$ 0.05} \\ $k{=}64\text{k}$ & 95.24 $\pm$ 0.13 & 95.98 $\pm$ 0.15 & 96.46 $\pm$ 0.08 & \textbf{96.51 $\pm$ 0.09} \\
\midrule
\multicolumn{5}{c}{\textit{$L=3$}} \\
\midrule
$k{=}8\text{k}$  & 95.96 $\pm$ 0.05 & 95.90 $\pm$ 0.11 & 95.33 $\pm$ 0.06 & \textbf{96.12 $\pm$ 0.02} \\ $k{=}16\text{k}$ & 97.06 $\pm$ 0.08 & 96.87 $\pm$ 0.13 & 96.39 $\pm$ 0.05 & \textbf{97.09 $\pm$ 0.07} \\ $k{=}32\text{k}$ & 97.43 $\pm$ 0.16 & 97.45 $\pm$ 0.05 & 97.32 $\pm$ 0.04 & \textbf{97.72 $\pm$ 0.04} \\ $k{=}64\text{k}$ & 97.66 $\pm$ 0.09 & 97.41 $\pm$ 0.07 & 97.69 $\pm$ 0.04 & \textbf{97.91 $\pm$ 0.09} \\
\midrule
\multicolumn{5}{c}{\textit{$L=6$}} \\
\midrule
$k{=}8\text{k}$  & 97.31 $\pm$ 0.01 & 97.09 $\pm$ 0.05 & 96.59 $\pm$ 0.09 & \textbf{97.37 $\pm$ 0.01} \\ $k{=}16\text{k}$ & \textbf{97.77 $\pm$ 0.09} & 97.63 $\pm$ 0.06 & 97.35 $\pm$ 0.04 & 97.68 $\pm$ 0.03 \\ $k{=}32\text{k}$ & 97.94 $\pm$ 0.06 & 97.89 $\pm$ 0.05 & 97.57 $\pm$ 0.10 & \textbf{98.13 $\pm$ 0.07} \\ $k{=}64\text{k}$ & 98.07 $\pm$ 0.04 & 98.02 $\pm$ 0.12 & 98.00 $\pm$ 0.04 & \textbf{98.21 $\pm$ 0.07} \\
\bottomrule
\end{tabular}
\end{table}

\paragraph{Observation.}
At $L{=}1$, Multilinear-CovJac overtakes Soft-Mix starting at $k{=}16\text{k}$ and the gap widens with width: $+0.26$pp at $k{=}16\text{k}$, $+0.58$pp at $k{=}32\text{k}$, $+1.27$pp at $k{=}64\text{k}$. Multilinear-STE tracks CovJac closely at $L{=}1$ (gap $<$0.1pp), consistent with the $c_{ab}$ starvation effect being mild at shallow depth. At $L{=}3$, CovJac matches Soft-Mix exactly at $k{=}16\text{k}$ while STE lags by 0.67pp, consistent with the coverage gap accumulating over layers.

\paragraph{Power-law fit.}
Fitting $A(k) = 100 - Bk^{-\alpha}$ at $L{=}1$ (log-linear regression on error vs width, $R^2 > 0.99$ for all methods; Table~\ref{tab:scaling-exponents}):

\begin{table}[t]
\centering
\caption{Power-law scaling exponents $A(k)=100-Bk^{-\alpha}$ fitted on MNIST $L=1$. \textbf{Bold} = highest $\alpha$ (steepest scaling).}
\label{tab:scaling-exponents}
\small
\begin{tabular}{lccc}
\toprule
Method & $\alpha$ & $B$ & Predicted error at $k{=}256\text{k}$ \\
\midrule
Soft-Mix           & 0.262 &  85 & 3.25\% \\
Gumbel-ST          & 0.350 & 196 & 2.52\% \\
M-STE (ours)    & \textbf{0.431} & 424 & 1.98\% \\
M-CovJac (ours) & \textbf{0.422} & 378 & 1.98\% \\
\bottomrule
\end{tabular}
\end{table}

Multilinear methods achieve scaling exponent $\alpha\approx 0.42$ versus Soft-Mix's $\alpha = 0.26$, a 61\% faster error decay with width. The practical consequence: to reach 1\% error, the power law predicts Multilinear needs ${\sim}230\text{k}$ neurons while Soft-Mix needs ${\sim}2.1\text{M}$, a $9\times$ width gap, compounded by the $4\times$ parameter-per-neuron efficiency.

\paragraph{Parameter-budget comparison.}
At matched parameter budget (4 params/neuron for CovJac vs 16 for
Soft-Mix), CovJac can be $4\times$ wider
(Table~\ref{tab:param-budget}):

\begin{table}[t]
\centering
\caption{Parameter-budget comparison on MNIST at matched total
parameters. CovJac uses $4\times$ wider networks than Soft-Mix at the same total parameter count.}
\label{tab:param-budget}
\small
\begin{tabular}{lcccc}
\toprule
Total params & CovJac ($k$) & Soft-Mix ($k$) & CovJac & Soft-Mix \\
\midrule
\multicolumn{5}{c}{\textit{$L=1$}} \\ 128k & $32\text{k}$ & $8\text{k}$ & \textbf{95.14\%} & 91.89\% \\ 256k & $64\text{k}$ & $16\text{k}$ & \textbf{96.51\%} & 93.22\% \\
\midrule
\multicolumn{5}{c}{\textit{$L=3$}} \\ 128k & $32\text{k}$ & $8\text{k}$ & \textbf{97.72\%} & 95.96\% \\ 256k & $64\text{k}$ & $16\text{k}$ & \textbf{98.21\%} & 97.06\% \\
\midrule
\multicolumn{5}{c}{\textit{$L=6$}} \\ 128k & $32\text{k}$ & $8\text{k}$ & \textbf{98.13\%} & 97.31\% \\ 256k & $64\text{k}$ & $16\text{k}$ & \textbf{98.21\%} & 97.77\% \\
\bottomrule
\end{tabular}
\end{table}

CovJac outperforms Soft-Mix at every matched budget: $+3.3$pp at $L{=}1$, $+1.2$pp at $L{=}3$, $+0.4$--$0.8$pp at $L{=}6$. The gap narrows at greater depth because accuracy approaches ceiling, but CovJac wins at every operating point.

\paragraph{Width scaling results.}
Multilinear methods achieve $\alpha > \alpha_{\text{Soft-Mix}}$ at $L{=}1$ ($0.42$ vs $0.26$, $R^2 > 0.99$). The CovJac--Soft-Mix gap widens monotonically with width, from $-0.27$pp at $k{=}8\text{k}$ to $+1.27$pp at $k{=}64\text{k}$, and the power-law fit predicts continued divergence.

\subsection{CIFAR-10 binary: full table}
\label{app:cifar-full}

Table~\ref{tab:cifar-full} reports the full CIFAR-10 results including gap decomposition.

\begin{table}[t]
\centering
\caption{CIFAR-10 gap decomposition ($k=128\text{k}$, $L=6$, 3 seeds).
DG = discretization gap (train-deploy mismatch); Gen = generalization gap (train-test on hard gates).}
\label{tab:cifar-full}
\begin{tabular}{lcccc}
\toprule
Method & Params/n & Hard Acc (\%) & DG (\%) & Gen (\%) \\
\midrule
Soft-Mix~\citep{petersen2022deep} & 16 & 58.13 $\pm$ 0.12 & 0.000 & 41.85 \\
Gumbel-ST                         & 16 & 57.91 $\pm$ 0.28 & 0.000 & 42.08 \\
M-STE (ours)            & 4  & 56.02 $\pm$ 0.23 & 0.000 & 43.33 \\
M-CovJac (ours) & 4 & \textbf{58.97 $\pm$ 0.26} & 0.071 & 40.96 \\
\bottomrule
\end{tabular}
\end{table}

\paragraph{Headline.}
Multilinear-CovJac beats Soft-Mix by 0.84pp with $4\times$ fewer parameters and near-zero discretization gap. Multilinear-STE is the worst-performing method on CIFAR-10; this is the empirical signature of $c_{ab}$ starvation at task scale (see \S\ref{app:h7}). All methods have essentially zero DG on binary inputs, so accuracy differences reflect optimization quality, not the train-to-deploy transition.

\subsection{Scaling to larger datasets and widths}
\label{app:scale-up}

Table~\ref{tab:scale-up} extends the comparison to CIFAR-100 at larger width ($k{=}1280\text{k}$, 100k iters) and to
Tiny-ImageNet~\citep{le2015tiny} (200 classes,
$64{\times}64$ images binarized with 31 thresholds per channel, yielding 380{,}928-dim binary inputs). CovJac remains the best method at every scale. The CovJac-vs-STE gap grows with width on both datasets: on CIFAR-100 from $+4.35$pp ($k{=}256\text{k}$) to $+5.02$pp ($k{=}1280\text{k}$); on
Tiny-ImageNet from $+1.52$pp ($k{=}512\text{k}$) to $+1.86$pp
($k{=}2560\text{k}$). This is consistent with the starvation mechanism becoming more consequential as the network has more capacity to exploit interactive gates.

\begin{table}[t]
\centering
\small
\caption{Scaling to larger datasets and widths ($L{=}6$, 3 seeds).
Last-10 accuracy (best-checkpoint in parentheses). \textbf{Bold} = best. CIFAR-100: 100k iters; Tiny-ImageNet: 50k iters (may benefit from longer training).}
\label{tab:scale-up}
\scriptsize
\begin{tabular}{llcccc}
\toprule
Dataset & $k$ & Soft-Mix & Gumbel-ST & M-STE (ours) & M-CovJac (ours) \\
\midrule
CIFAR-100 & 256k & 27.92{\tiny$\pm$.43} (28.21{\tiny$\pm$.38}) & 27.72{\tiny$\pm$.40} (28.08{\tiny$\pm$.33}) & 24.02{\tiny$\pm$.17} (25.02{\tiny$\pm$.16}) & \textbf{28.37{\tiny$\pm$.22}} (28.91{\tiny$\pm$.14}) \\ CIFAR-100 & 1280k & 30.56{\tiny$\pm$.31} (30.87{\tiny$\pm$.26}) & 31.65{\tiny$\pm$.20} (31.88{\tiny$\pm$.17}) & 27.70{\tiny$\pm$.31} (29.77{\tiny$\pm$.13}) & \textbf{32.72{\tiny$\pm$.09}} (32.88{\tiny$\pm$.08}) \\
\midrule
Tiny-Im. & 512k & 7.46{\tiny$\pm$.20} (7.74{\tiny$\pm$.12}) & 7.41{\tiny$\pm$.05} (7.67{\tiny$\pm$.08}) & 6.49{\tiny$\pm$.08} (7.57{\tiny$\pm$.10}) & \textbf{8.01{\tiny$\pm$.13}} (8.25{\tiny$\pm$.17}) \\
Tiny-Im. & 2560k & 9.18{\tiny$\pm$.04} (9.43{\tiny$\pm$.06}) & 9.10{\tiny$\pm$.21} (9.44{\tiny$\pm$.29}) & 8.57{\tiny$\pm$.02} (9.26{\tiny$\pm$.13}) & \textbf{10.43{\tiny$\pm$.12}} (10.64{\tiny$\pm$.13}) \\
\bottomrule
\end{tabular}
\end{table}

\subsection{Gradient-norm diagnostic}
\label{app:gradnorm}

For each method we record, at every evaluation checkpoint, the per-coefficient gradient $L_2$ norm averaged over the training minibatches within that evaluation window. For Multilinear methods the four components are $(\|\partial\mathcal{L}/\partial c_0\|, \|\partial\mathcal{L}/\partial c_a\|, \|\partial\mathcal{L}/\partial c_b\|, \|\partial\mathcal{L}/\partial c_{ab}\|)$. We then compute the coverage ratio $r = \|\partial\mathcal{L}/\partial c_{ab}\|/\|\partial\mathcal{L}/\partial c_0\|$.

\paragraph{Result.}
Multilinear-STE: ratio in $[0.19, 0.25]$ throughout training, mean 0.227, consistent with Lemma~\ref{lem:coverage}'s $\E[ab] = 1/4$ prediction up to ${\sim}10\%$ downward correction from cross-layer correlations between the upstream error $\delta$ and the input bits $(a,b)$. Multilinear-CovJac: at uniform $\bpi$ the closed-form prediction is $r = 4$ (because under uniform $\bpi$ only the $(a{=}b{=}0)$ row of $R_0$ is non-zero while all four rows of $R_3$ are non-zero); empirically the ratio is $\sim 1$ at iter 1k and grows monotonically to a network-average plateau of approximately 7.3 by iteration 25k, then stabilizes. The growth is driven by the diagonal entry $J_{00} = (2/\tau)\Var_\pi(c_0)$ collapsing as $\bpi$ concentrates onto gates with $c_{ab}\neq 0$, while the off-diagonal $J_{03}$ remains comparable in magnitude and keeps $c_{ab}$ fed through the always-active $\psi_0{=}1$ channel; on neurons that commit to $c_{ab}{=}0$ gates the same mechanism drives the ratio to a floor near 0.8.

\subsection{Method semantics: additional notes}
\label{app:method-semantics}

\paragraph{CovJac derivation and VQ mapping.}
We derive $\mathbf{J} = (2/\tau)\mathrm{Cov}_{\boldsymbol{\omega}}(\bG)$ as the standard soft-VQ Jacobian~\citep{agustsson2017softhard}; Proposition~\ref{prop:covjac}'s input-independence analysis is novel. Mapping: Multilinear-STE $=$ VQ-VAE hard quantize + STE~\citep{oord2017neural}; Multilinear-CovJac $=$ soft-to-hard VQ. Codebook (16 Boolean gates) and analysis are novel; quantization techniques are established.

\paragraph{Skip connection and identity gate analysis.}
Identity gates ($a$, $b$) act as implicit skip connections. They can, in principle, carry gradient past a committed layer, but in Soft-Mix the $\bpi$-routing for identity gates is itself subject to the Proposition~\ref{prop:softmix-cancel} cancellation at uniform $\bpi$.

\subsection{Interaction-task scaling}
\label{app:h7}

\paragraph{Claim.}
The Multilinear-CovJac vs Multilinear-STE accuracy gap reflects the density of interactive ($c_{ab}\neq 0$) gates needed to fit the task. The gap scales monotonically with task interaction density across seven datasets.

\begin{table}[t]
\centering
\caption{Splice and MONK's-2 cross-dataset comparison (last-10 accuracy $\pm$ std, 3 seeds).
\textbf{Bold} = best per dataset. CovJac$-$STE gap shown for interaction-density scaling.}
\label{tab:h7-cross-dataset}
\small
\begin{tabular}{@{}lcccccc@{}}
\toprule
Dataset & Soft-Mix & Gumbel-ST & M-STE (ours) & M-CovJac (ours) & CovJac$-$STE \\
\midrule
Splice ($k{=}4095$)  & 97.01$\pm$0.11 & 96.81$\pm$0.30 & 97.07$\pm$0.00 & \textbf{97.20$\pm$0.23} & +0.13pp \\ MONK's-2 ($k{=}136$) & 81.32$\pm$1.09 & 81.46$\pm$0.89 & 78.53$\pm$1.25 & \textbf{86.06$\pm$2.17} & +7.53pp \\
\bottomrule
\end{tabular}
\end{table}
\noindent{\small All rows are $L{=}6$, 3-seed means $\pm$ std (10k iters for MONK's-2). MONK's-2 uses unique wiring ($k{=}136$
is the maximum for 17 input features).}

The gap is monotonically increasing with interaction density: on Adult ($\tau{=}0.1$) the gap is negligible (+0.02pp), CovJac overtakes on images (+2.22pp on SVHN, +2.95pp on CIFAR-10), the advantage widens to +4.35pp on CIFAR-100, and reaches +7.53pp on the synthetic MONK's-2 benchmark (which requires all pairwise interactions). Splice (+0.13pp, 3 seeds) slots between Adult and MNIST, consistent with moderate interaction density in DNA sequence patterns. All datasets use binary inputs, the same architecture ($L{=}6$, GroupSum readout), and the same training protocol (Adam, 3 seeds). The only controlled difference is the task.

\paragraph{Why interaction-dense tasks amplify the gap.}
The 16 Boolean gates partition by $|c_{ab}|$: 6 gates have $c_{ab}{=}0$ (constants, projections, negations, i.e.\ functions linear in individual inputs on the Boolean cube), 8 have $|c_{ab}|{=}1$ (AND/OR family, implications), and 2 have $|c_{ab}|{=}2$ (XOR, XNOR). The non-$c_{ab}$ gates suffice for near-linearly-separable problems like MNIST binary. Natural-image classification requires many XOR/implication-like compositions; a network whose $c_{ab}$ coordinate is under-trained cannot express these gates efficiently. On Adult, the overhead of CovJac's coupling outweighs the benefit because few neurons need interactive gates.

\paragraph{Three concrete numerical signatures.}
\textbf{(1)} Multilinear-CovJac beats Soft-Mix with $4\times$ fewer parameters on CIFAR-10 (+0.84pp) and CIFAR-100 (+0.45pp). On MNIST both are tied (98.30\%). The advantage only emerges when the task stresses the $c_{ab}$ coordinate. \textbf{(2)} Multilinear-STE is the worst method on CIFAR-10/100 (56.02\% / 24.02\%), underperforming even Soft-Mix. \textbf{(3)} Cross-coefficient coupling scales with task demand. Proposition~\ref{prop:covjac}'s $J_{03}$ off-diagonal routes gradient into $c_{ab}$ on every sample. On MNIST the baseline STE coverage (25\%) is enough, so CovJac yields only +0.21pp; on CIFAR-100 it is not, so CovJac yields +4.35pp; on MONK's-2 the gap reaches +7.53pp.

\paragraph{Formalizing interaction demand.}
\label{app:interaction-demand}
The ``interaction demand'' of a task can be operationalized as the \emph{$c_{ab}$ starvation severity}: the accuracy deficit that Multilinear-STE suffers relative to Soft-Mix, i.e.\ $\Delta_{\text{starv}} = A^{\text{SM}} - A^{\text{STE}}$. This measures how much the 25\% $c_{ab}$ coverage (Lemma~\ref{lem:coverage}) hurts on a given task, using only baseline methods (independent of CovJac). The CovJac advantage $\Delta_{\text{CJ}} = A^{\text{CJ}} - A^{\text{STE}}$ should correlate with $\Delta_{\text{starv}}$ if CovJac's mechanism (coupling $c_{ab}$ to the always-active channel) specifically addresses the starvation bottleneck.

\begin{table}[t]
\centering
\small
\caption{Starvation severity vs CovJac advantage across 7 datasets.
$\Delta_{\text{starv}}$: Soft-Mix $-$ M-STE (baseline deficit). $\Delta_{\text{CJ}}$: M-CovJac $-$ M-STE (our advantage). Pearson $r{=}0.85$. Ratio undefined (---) when $\Delta_{\text{starv}} \leq 0$.}
\label{tab:starvation-severity}
\begin{tabular}{lrrr}
\toprule
Dataset & $\Delta_{\text{starv}}$ & $\Delta_{\text{CJ}}$ & Ratio \\
\midrule
Adult          & $-0.23$ & $+0.02$ & --- \\ Splice         & $-0.06$ & $+0.13$ & --- \\ MNIST          & $+0.20$ & $+0.21$ & 1.05$\times$ \\ SVHN           & $+1.52$ & $+2.22$ & 1.46$\times$ \\ CIFAR-10       & $+2.11$ & $+2.95$ & 1.40$\times$ \\ CIFAR-100      & $+3.90$ & $+4.35$ & 1.12$\times$ \\ MONK's-2       & $+2.79$ & $+7.53$ & 2.70$\times$ \\
\bottomrule
\end{tabular}
\end{table}

Table~\ref{tab:starvation-severity} confirms the correlation ($r{=}0.85$): on datasets where starvation severity is negative or near zero (Adult, Splice), CovJac provides no benefit; where the severity is positive and large (CIFAR-10/100, MONK's-2), CovJac's advantage scales proportionally. The outlier is MONK's-2 (ratio $2.70\times$), where the target function provably requires all pairwise interactions; CovJac's 100\% $c_{ab}$ coverage has outsized benefit on this maximally interactive task.

\subsection{Depth scaling}
\label{app:depth-h6}

\paragraph{Claim.}
Multilinear methods (and, more generally, any single-gate-forward method) scale to $L{=}12$ where Soft-Mix stalls. The backward-cancellation theory (Proposition~\ref{prop:softmix-cancel} + Observation~1) predicts the slowdown (Table~\ref{tab:depth}).

\begin{table}[t]
\centering
\small
\caption{Depth scaling on MNIST ($k=64\text{k}$, 50k iters,
3 seeds). Last-10 accuracy (best-checkpoint in parentheses). \textbf{Bold} = best per column.}
\label{tab:depth}
\scriptsize
\begin{tabular}{lcccc}
\toprule
Method & $L=3$ & $L=6$ & $L=9$ & $L=12$ \\
\midrule
Soft-Mix    & 97.80{\tiny$\pm$.03} (97.88{\tiny$\pm$.02}) & 98.29{\tiny$\pm$.04} (98.33{\tiny$\pm$.03}) & 97.34{\tiny$\pm$.16} (97.49{\tiny$\pm$.10}) & 82.94{\tiny$\pm$2.85} (86.58{\tiny$\pm$2.16}) \\
Gumbel-ST   & 97.80{\tiny$\pm$.01} (97.86{\tiny$\pm$.01}) & 98.22{\tiny$\pm$.04} (98.27{\tiny$\pm$.03}) & 98.19{\tiny$\pm$.03} (98.24{\tiny$\pm$.04}) & 97.89{\tiny$\pm$.03} (97.98{\tiny$\pm$.06}) \\
M-STE       & 97.68{\tiny$\pm$.04} (97.77{\tiny$\pm$.04}) & 98.04{\tiny$\pm$.06} (98.15{\tiny$\pm$.05}) & 98.01{\tiny$\pm$.06} (98.13{\tiny$\pm$.03}) & 97.99{\tiny$\pm$.05} (98.08{\tiny$\pm$.04}) \\
M-CovJac    & \textbf{97.84{\tiny$\pm$.03}} (97.89{\tiny$\pm$.05}) & \textbf{98.29{\tiny$\pm$.04}} (98.36{\tiny$\pm$.04}) & \textbf{98.31{\tiny$\pm$.06}} (98.35{\tiny$\pm$.05}) & \textbf{98.28{\tiny$\pm$.06}} (98.34{\tiny$\pm$.08}) \\
\bottomrule
\end{tabular}
\end{table}

\paragraph{Headline.}
At $L{=}12$, Soft-Mix stalls at 82.94\% $\pm$ 2.85, a \textbf{15.4pp drop} from its $L{=}6$ value of 98.29\%. All three single-gate-forward methods (Multilinear-CovJac 98.28\%, Multilinear-STE 97.99\%, Gumbel-ST 97.89\%) remain above 97.8\% at $L{=}12$. Multilinear-CovJac is the best method at every depth tested. The Multilinear-STE $\to$ Multilinear-CovJac gap is depth-stable at $\approx 0.3$pp.

\paragraph{Extended training at $L{=}12$.}
To verify that Soft-Mix's depth penalty is persistent rather than a transient lag, we double the training budget to 100k iterations at $L{=}12$, $k{=}64\text{k}$ on MNIST (Table~\ref{tab:depth-100k}).

\begin{table}[t]
\centering
\small
\caption{MNIST depth $L{=}12$ at 100k iterations ($k=64\text{k}$, 3 seeds).
Last-10 accuracy. \textbf{Bold} = best.}
\label{tab:depth-100k}
\begin{tabular}{lcc}
\toprule
Method & 50k iters & 100k iters \\
\midrule
Soft-Mix           & 82.94 $\pm$ 2.85 & 92.08 $\pm$ 0.45 \\
Gumbel-ST          & 97.89 $\pm$ 0.03 & 98.04 $\pm$ 0.05 \\
M-STE (ours)    & 97.99 $\pm$ 0.05 & 98.05 $\pm$ 0.11 \\
M-CovJac (ours) & \textbf{98.30 $\pm$ 0.07} & \textbf{98.26 $\pm$ 0.07} \\
\bottomrule
\end{tabular}
\end{table}

Soft-Mix recovers from 82.94\% to 92.08\% but remains 6.0pp below all
single-gate-forward methods, confirming that the backward cancellation imposes a persistent accuracy ceiling, not a training-speed penalty.

\paragraph{Single-gate-forward dichotomy.}
The three methods that commit to a single gate per forward pass all survive at depth; the one method that sums over 16 gates during forward is the only one that stalls. Gumbel-ST is particularly diagnostic: it has the same 16-parameter count as Soft-Mix, yet its Gumbel-noise sampling produces a one-hot selection per forward, avoiding the Lemma~\ref{lem:zerosum} zero-sums inside $\partial z/\partial a$. Its survival at $L{=}12$ (97.89\%) is direct empirical confirmation that the Proposition~\ref{prop:softmix-cancel} backward cancellation, not parameter count, is the operative mechanism behind Soft-Mix's depth penalty.

\paragraph{Cross-dataset: Adult depth sweep.}
Adult ($k{=}4096$, $L\in\{3,6,9,12\}$, 50k iters, 3 seeds; Table~\ref{tab:adult-depth}):

\begin{table}[t]
\centering
\caption{Adult depth scaling ($k=4096$, $\tau{=}0.1$ for CovJac, 3 seeds).
Last-10 accuracy (best-checkpoint in parentheses). \textbf{Bold} = best per column.}
\label{tab:adult-depth}
\scriptsize
\begin{tabular}{lcccc}
\toprule
Method & $L=3$ & $L=6$ & $L=9$ & $L=12$ \\
\midrule
Soft-Mix    & 84.80{\tiny$\pm$.04} (84.90{\tiny$\pm$.01}) & 84.68{\tiny$\pm$.05} (85.01{\tiny$\pm$.03}) & 84.65{\tiny$\pm$.05} (84.96{\tiny$\pm$.02}) & 82.36{\tiny$\pm$.84} (83.37{\tiny$\pm$.79}) \\
Gumbel-ST   & 84.83{\tiny$\pm$.02} (84.97{\tiny$\pm$.05}) & 84.75{\tiny$\pm$.02} (85.01{\tiny$\pm$.01}) & 84.70{\tiny$\pm$.07} (84.97{\tiny$\pm$.02}) & 84.57{\tiny$\pm$.10} (84.93{\tiny$\pm$.09}) \\
M-STE       & \textbf{84.99{\tiny$\pm$.03}} (85.08{\tiny$\pm$.03}) & 84.90{\tiny$\pm$.05} (85.13{\tiny$\pm$.03}) & 84.70{\tiny$\pm$.06} (85.08{\tiny$\pm$.03}) & 84.58{\tiny$\pm$.02} (85.07{\tiny$\pm$.06}) \\
M-CovJac$^\dagger$ & 84.47{\tiny$\pm$.02} (84.51{\tiny$\pm$.02}) & \textbf{84.93{\tiny$\pm$.03}} (84.96{\tiny$\pm$.02}) & \textbf{84.98{\tiny$\pm$.04}} (85.10{\tiny$\pm$.05}) & \textbf{84.91{\tiny$\pm$.08}} (85.04{\tiny$\pm$.06}) \\
\bottomrule
\end{tabular}

{\tiny $^\dagger$$\tau{=}0.1$ (same as Table~\ref{tab:mnist-main}).}
\end{table}

On Adult, Soft-Mix shows the largest depth penalty ($-2.44$pp from $L{=}3$ to $L{=}12$), though milder than on MNIST ($-15.4$pp). With $\tau{=}0.1$, CovJac is the most depth-stable method ($-0.02$pp from $L{=}6$ to $L{=}12$), reversing the degradation observed at $\tau{=}1.0$ ($-1.08$pp). The reduced coupling strength at low $\tau$ eliminates the overhead that accumulates with depth on this low-interaction dataset.

\paragraph{Cross-dataset: CIFAR-10 depth sweep.}
CIFAR-10 ($k{=}128\text{k}$, $L\in\{3,6,9,12\}$, 50k iters, 3 seeds; Table~\ref{tab:cifar-depth}):

\begin{table}[t]
\centering
\caption{CIFAR-10 depth scaling ($k=128\text{k}$, 50k iters, 3 seeds).
Last-10 accuracy (best-checkpoint in parentheses). \textbf{Bold} = best per column.}
\label{tab:cifar-depth}
\scriptsize
\begin{tabular}{lcccc}
\toprule
Method & $L=3$ & $L=6$ & $L=9$ & $L=12$ \\
\midrule
Soft-Mix    & 57.37{\tiny$\pm$.19} (57.60{\tiny$\pm$.24}) & 58.13{\tiny$\pm$.12} (58.31{\tiny$\pm$.09}) & 54.76{\tiny$\pm$.20} (55.06{\tiny$\pm$.18}) & 20.79{\tiny$\pm$.21} (22.98{\tiny$\pm$.18}) \\
Gumbel-ST   & 57.00{\tiny$\pm$.41} (57.29{\tiny$\pm$.35}) & 58.24{\tiny$\pm$.25} (58.47{\tiny$\pm$.33}) & 56.42{\tiny$\pm$.05} (56.92{\tiny$\pm$.13}) & 51.66{\tiny$\pm$.18} (53.49{\tiny$\pm$.52}) \\
M-STE       & 56.51{\tiny$\pm$.21} (56.85{\tiny$\pm$.25}) & 56.02{\tiny$\pm$.23} (56.53{\tiny$\pm$.12}) & 54.66{\tiny$\pm$.08} (55.64{\tiny$\pm$.20}) & 53.34{\tiny$\pm$.21} (55.08{\tiny$\pm$.25}) \\
M-CovJac    & \textbf{57.44{\tiny$\pm$.26}} (57.75{\tiny$\pm$.25}) & \textbf{58.97{\tiny$\pm$.26}} (59.18{\tiny$\pm$.26}) & \textbf{58.41{\tiny$\pm$.12}} (58.88{\tiny$\pm$.13}) & \textbf{58.44{\tiny$\pm$.18}} (58.75{\tiny$\pm$.29}) \\
\bottomrule
\end{tabular}
\end{table}

On CIFAR-10, Soft-Mix collapses to 20.79\% at $L{=}12$ ($-37.3$pp from $L{=}6$), confirming that the MNIST depth penalty generalizes to harder tasks. Multilinear-CovJac degrades by only $-0.53$pp, and is the only method that holds above 58\% across all depths tested. Multilinear-STE degrades modestly ($-2.68$pp), consistent with $c_{ab}$ starvation accumulating over depth on an interaction-dense task.

\paragraph{Cross-dataset: CIFAR-100 depth sweep.}
CIFAR-100 ($k{=}256\text{k}$, $L\in\{3,6,9,12\}$, 100k iters, 3 seeds; Table~\ref{tab:cifar100-depth}):

\begin{table}[t]
\centering
\caption{CIFAR-100 depth scaling ($k=256\text{k}$, 100k iters, 3 seeds).
Last-10 accuracy (best-checkpoint in parentheses). \textbf{Bold} = best per column.}
\label{tab:cifar100-depth}
\scriptsize
\begin{tabular}{lcccc}
\toprule
Method & $L=3$ & $L=6$ & $L=9$ & $L=12$ \\
\midrule
Soft-Mix    & 27.33{\tiny$\pm$.17} (27.58{\tiny$\pm$.20}) & 27.92{\tiny$\pm$.43} (28.21{\tiny$\pm$.38}) & 24.83{\tiny$\pm$.24} (25.13{\tiny$\pm$.17}) & 4.59{\tiny$\pm$.46} (5.25{\tiny$\pm$.74}) \\
Gumbel-ST   & 27.11{\tiny$\pm$.23} (27.40{\tiny$\pm$.17}) & 27.72{\tiny$\pm$.40} (28.08{\tiny$\pm$.33}) & 27.16{\tiny$\pm$.07} (27.54{\tiny$\pm$.14}) & 26.20{\tiny$\pm$.27} (26.44{\tiny$\pm$.25}) \\
M-STE       & 25.30{\tiny$\pm$.21} (25.71{\tiny$\pm$.22}) & 24.02{\tiny$\pm$.17} (25.02{\tiny$\pm$.16}) & 22.60{\tiny$\pm$.09} (24.29{\tiny$\pm$.12}) & 20.83{\tiny$\pm$.13} (23.97{\tiny$\pm$.13}) \\
M-CovJac    & \textbf{26.98{\tiny$\pm$.22}} (27.24{\tiny$\pm$.20}) & \textbf{28.37{\tiny$\pm$.22}} (28.91{\tiny$\pm$.14}) & \textbf{27.74{\tiny$\pm$.18}} (28.30{\tiny$\pm$.15}) & \textbf{27.61{\tiny$\pm$.19}} (28.15{\tiny$\pm$.18}) \\
\bottomrule
\end{tabular}
\end{table}

On CIFAR-100, the Soft-Mix depth collapse is catastrophic: from 27.92\% at $L{=}6$ to 4.59\% at $L{=}12$ ($-23.3$pp), approaching random chance for 100 classes (1\%). CovJac holds at 27.61\% ($-0.76$pp from $L{=}6$). M-STE degrades more severely on CIFAR-100 than on CIFAR-10 ($-3.19$pp vs $-2.68$pp), consistent with the interaction-demand scaling: the 100-class task requires more interactive gates, amplifying the $c_{ab}$ starvation effect.

\subsection{Gate codebook and per-neuron coefficient distribution}
\label{app:codebook-neurons}

Figure~\ref{fig:codebook-3d} shows the 16-gate codebook in 3D coefficient space. The per-neuron density plots are in the main text (Figure~\ref{fig:codebook-neurons}).

\begin{figure}[t]
\centering
\includegraphics[width=0.65\linewidth]{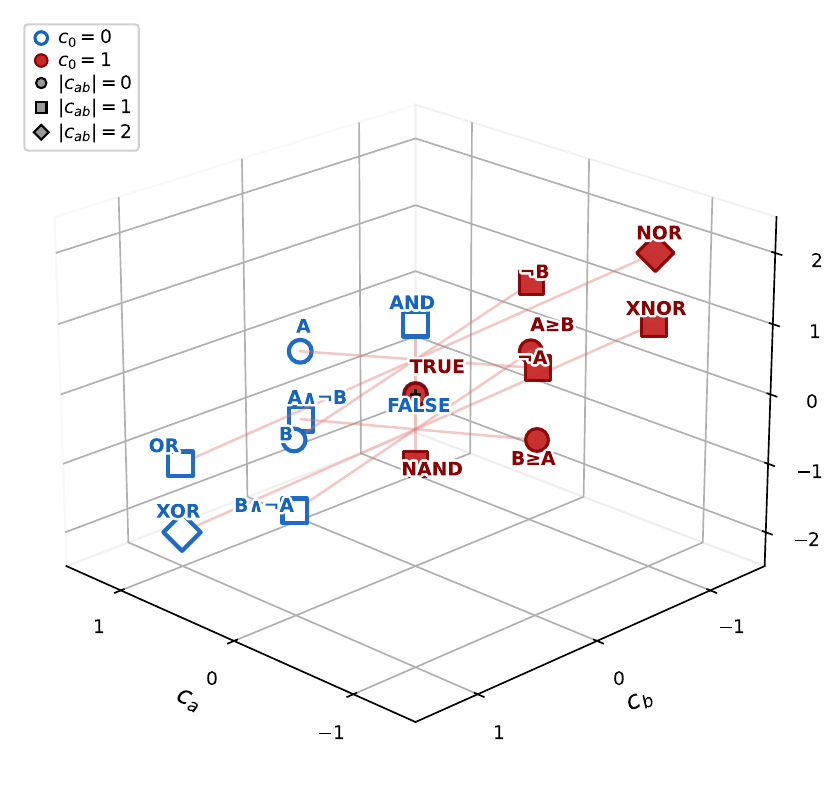}
\caption{The 16-gate codebook in $(c_a, c_b, c_{ab})$ space.
Blue ($c_0{=}0$): FALSE, A, B, AND, OR, XOR and variants. Red ($c_0{=}1$): TRUE, $\neg$A, $\neg$B, NAND, NOR, XNOR and variants. Lines connect complementary pairs ($g \leftrightarrow 1{-}g$). Shape indicates $|c_{ab}|$: $\circ$ separable, $\square$ weak, $\diamond$ strong interaction.}
\label{fig:codebook-3d}
\end{figure}

\paragraph{Snapped gate assignments.}
Figure~\ref{fig:gate-selection} shows the per-neuron gate selection at convergence on CIFAR-10. Each method learns continuous coefficients $\bc\in\R^4$ that are snapped to the nearest codebook entry $\bch\in\bG$ via Eq.~\ref{eq:snap}. The key observation: M-STE continuous coefficients drift far from the codebook (std$(c_{ab})$=52.6, range $\pm 240$) because STE provides no force pulling $\bc$ back; the snapping operation simply picks the nearest of 16 fixed integer points regardless of distance. CovJac's soft-VQ proximity weighting creates an implicit restoring force, keeping coefficients closer to the codebook (std=5.9). Despite the drift, M-STE's \emph{snapped} gate distribution is well-distributed: 20.7\% of neurons select strong interaction gates ($|c_{ab}|{\geq}2$), vs 25.7\% for CovJac and 17.8\% for Soft-Mix.

\begin{figure}[t]
\centering
\includegraphics[width=\linewidth]{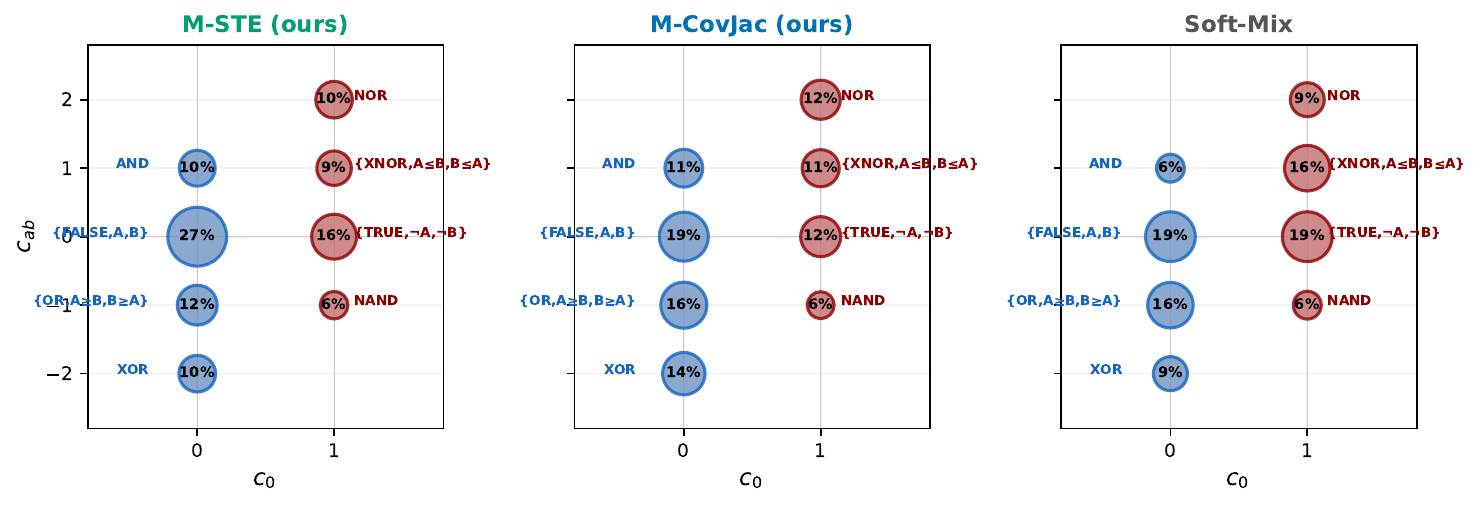}
\caption{Snapped gate assignments on CIFAR-10 ($k{=}128\text{k}$,
$L{=}6$). Bubble size $\propto$ fraction of neurons selecting each gate. CovJac selects more XOR/XNOR (strong interaction) gates;
Soft-Mix clusters near separable gates.}
\label{fig:gate-selection}
\end{figure}

\paragraph{Gate type distribution across datasets.}
The gate type composition (constant, separable, weak interaction $|c_{ab}|{=}1$, and strong interaction $|c_{ab}|{=}2$) for all four methods across Adult, MNIST, and CIFAR-10 is shown in Figure~\ref{fig:codebook-neurons} (main text). The 16 Boolean gates partition into: 2 constant (FALSE, TRUE), 4 separable (A, B, $\neg$A, $\neg$B), 8 weak interaction (AND, OR, NAND, NOR, implications), and 2 strong interaction (XOR, XNOR).

CovJac consistently selects the most strong interaction gates: 36.2\% on Adult, 27.2\% on MNIST, 25.7\% on CIFAR-10, compared to 16--21\% for other methods. This confirms that the $J_{03}$ coupling enables neurons to learn $c_{ab}$ and reach XOR/XNOR gates that STE neurons cannot access due to 25\% gradient coverage.

CovJac also has the fewest constant gates ($<$0.5\% vs 4--10\% for others), indicating almost no wasted neurons. The reason is structural: the $J_{03}$ off-diagonal couples $c_{ab}$ to the constant channel on every sample, so neurons receive gradient that actively pushes them away from the constant region of the codebook ($c_a{=}c_b{=}c_{ab}{=}0$). Under STE, the 25\% coverage for $c_{ab}$ means some neurons never receive enough signal to move away from their initial nearest gate, which is often constant or separable; Soft-Mix's gradient cancellation has a similar effect. On Adult (separable task), CovJac's high strong-interaction fraction (36\%) explains its slight accuracy gap at $\tau{=}1$: the network assigns interactive gates where separable ones would suffice.

\subsection{Temperature sensitivity (\texorpdfstring{$\tau$}{tau})}
\label{app:tau}

Multilinear-CovJac uses a temperature parameter $\tau$ in the softmax that converts distances to selection weights (Section~\ref{sec:method}). Low $\tau$ sharpens the distribution toward a hard one-hot selection; high $\tau$ flattens it toward uniform. We sweep $\tau$ over two orders of magnitude on MNIST binary ($k{=}64\text{k}$, $L{=}6$, 50k iterations, 3 seeds each; Table~\ref{tab:tau-sensitivity}).

\begin{table}[t]
\centering
\caption{Temperature sensitivity of Multilinear-CovJac on MNIST ($k=64\text{k}$, $L=6$, 3 seeds).}
\label{tab:tau-sensitivity}
\small
\begin{tabular}{lc}
\toprule
$\tau$ & Last-10 Hard Acc.\ (\%) \\
\midrule
0.1  & 96.99 \\ 0.3  & 98.11 \\ 0.5  & 98.17 \\ 1.0  & 98.30 \\ 2.0  & \textbf{98.33} \\ 5.0  & 98.23 \\ 10.0 & 98.08 \\
\bottomrule
\end{tabular}
\end{table}

\paragraph{Robustness.}
Across the range $\tau \in [0.3, 5.0]$, all accuracies fall within 0.22pp of the peak (98.33\% at $\tau{=}2.0$), confirming that Multilinear-CovJac is not brittle with respect to this hyperparameter. Only the extreme value $\tau{=}0.1$ underperforms meaningfully ($-1.34$pp versus peak): at this temperature the softmax is nearly hard, so early training receives very sparse gradients and the network converges to a weaker local minimum.

\paragraph{Default choice.}
We use $\tau{=}1.0$ throughout all experiments in this paper. The sweep validates this default: $\tau{=}1.0$ achieves 98.30\%, only 0.03pp below peak, and sits comfortably in the middle of the robust plateau. No per-dataset tuning of $\tau$ is necessary.

\paragraph{$\tau$ sensitivity on Adult.}
On the separable Adult task ($k{=}4096$, $L{=}6$, 3 seeds), low $\tau$ recovers M-STE-level performance: $\tau{=}0.1$ achieves 84.93\%, matching M-STE (84.91\%) and surpassing fixed $\tau{=}1.0$ CovJac (84.57\%). This confirms that $\tau$ controls the coupling strength: low $\tau$ reduces the $J_{03}$ coupling that adds overhead on separable tasks.

\paragraph{Learnable $\tau$ (negative result).}
We test whether $\tau$ can be learned per layer, adapting coupling strength to task demands, by making $\tau$ a learnable parameter (initialized via $\log\tau$, optimized jointly with Adam). On Adult ($k{=}4096$), learnable $\tau$ (init 1.0) achieves 84.90\%, matching M-STE. However, on MNIST ($k{=}64\text{k}$) accuracy drops to 98.03\% (vs 98.30\% fixed), and on CIFAR-10 ($k{=}128\text{k}$) it drops to 56.12$\pm$0.84\% (vs 58.97\% fixed, 3 seeds). The learned $\tau$ collapses to ${\sim}0.01$--$0.03$ per layer, with high seed variance (some layers sporadically retain $\tau{\approx}0.4$ but most collapse), effectively reducing CovJac to STE and losing the $c_{ab}$ coupling. The collapse occurs because the gradient signal favors sharper soft-VQ (lower noise), but this simultaneously removes the cross-coefficient coupling that makes CovJac effective. Fixed $\tau{=}1$ remains the robust default.


\section{Basis--Constraint Factorial Experiment}
\label{app:basis-constraint}

\subsection{Design}
\label{app:g-design}

Factorial over:
\begin{itemize}
\item \textbf{basis} $\in\{$canonical polynomial $(1,a,b,ab)$, corner
$\phi_{ij}(a,b)\}$,
\item \textbf{constraint} $\in\{$free (sigmoid, no snap), STE ($L_2$ snap
+ straight-through gradient)$\}$,
\end{itemize}
yielding four cells: Multilinear-free, Multilinear-STE, IWP-free,
IWP-STE. Depths $L\in\{1,6,12\}$, $k{=}64\text{k}$, MNIST-bin, 3 seeds,
standard Adam protocol from Appendix~\ref{app:hyperparams}. Results are reported in Table~\ref{tab:basis-constraint-main}.

\subsection{Main result table ($L=6$, $k=64\text{k}$)}
\label{app:g-main}

\begin{table}[t]
\centering
\caption{Basis--constraint factorial ($L=6$, $k=64\text{k}$, MNIST, 3 seeds). \textbf{Bold} = best per group (STE / free).}
\label{tab:basis-constraint-main}
\begin{tabular}{lc}
\toprule
Cell & Hard Acc (\%) \\
\midrule
Multilinear-STE (poly\_snap)                        & \textbf{97.98 $\pm$ 0.11} \\
IWP-STE (iwp\_ste, our ablation)                    & 82.94 $\pm$ 1.59 \\
\midrule
\textit{Multilinear-free$^*$ (poly\_free)}          & \textit{98.11 $\pm$ 0.04} \\
\textit{IWP-free$^*$ (iwp\_free)}                   & \textit{\textbf{98.37 $\pm$ 0.04}} \\
IWP~\citep{ruttgers2026narrowing} original (reference) & 77.95 $\pm$ 2.60 \\
\bottomrule
\end{tabular}

{\tiny $^*$Free methods use sigmoid activations and are not Boolean
logic gates at deployment.}
\end{table}

\paragraph{Reading the table.}
(i) Under the STE constraint, the basis matters \emph{enormously}: Multilinear-STE beats IWP-STE by \textbf{15.04pp} (97.98 vs 82.94; the main-text ablation Table~\ref{tab:ablation} reports 17.25pp from independent seeds). (ii) Without the STE constraint (free methods, italic in the table), the two bases are comparable: IWP-free 98.37\% vs Multilinear-free 98.06\%, confirming they span the same function space (Appendix~\ref{app:bijection}). Note that free methods use sigmoid activations and produce \emph{continuous} outputs; they are not Boolean logic gates and cannot be deployed as combinational circuits. Their role here is purely as an ablation control: by removing the discrete constraint, we isolate the basis effect from the STE effect. The significant basis $\times$ constraint interaction confirms that the multilinear advantage is \textbf{STE-specific}, not a function-class property.

\subsection{Depth sweep (across $L\in\{1,6,12\}$)}
\label{app:g-depth}

The STE-specific gap is stable across depths and widens dramatically at $L{=}12$. All results $k{=}64\text{k}$, 3 seeds, last-10 hard accuracy.

\begin{table}[t]
\centering
\caption{Basis--constraint depth sweep ($k=64\text{k}$, MNIST, 3 seeds). Last-10 accuracy. \textbf{Bold} = best per column within each group (STE / free).}
\label{tab:basis-constraint-depth}
\begin{tabular}{lccc}
\toprule
Cell & $L=1$ & $L=6$ & $L=12$ \\
\midrule
Multilinear-STE  & \textbf{95.59 $\pm$ 0.09} & \textbf{97.98 $\pm$ 0.11} & \textbf{98.03 $\pm$ 0.10} \\
IWP-STE          & 95.31 $\pm$ 0.13 & 82.94 $\pm$ 1.59 & 42.76 $\pm$ 2.73 \\
\midrule
\textit{Multilinear-free$^*$} & \textit{95.03 $\pm$ 0.11} & \textit{98.11 $\pm$ 0.04} & \textit{92.95 $\pm$ 0.27} \\
\textit{IWP-free$^*$}         & \textit{95.40 $\pm$ 0.10} & \textit{\textbf{98.37 $\pm$ 0.04}} & \textit{\textbf{97.80 $\pm$ 0.13}} \\
\bottomrule
\end{tabular}

{\tiny $^*$Not Boolean logic gates at deployment (sigmoid, no quantization).}
\end{table}

\paragraph{Results.}
Three observations emerge from Table~\ref{tab:basis-constraint-depth}. \textbf{(1)}~The STE-specific basis gap widens with depth: at $L{=}1$ the two STE methods differ by only 0.28pp (95.59 vs 95.31), but at $L{=}12$ the gap reaches 55.27pp (98.03 vs 42.76) as IWP-STE collapses entirely. \textbf{(2)}~Without the STE constraint (free methods, italic rows), the ranking \emph{reverses} at $L{=}12$: IWP-free (97.80\%) holds steady while Multilinear-free drops to 92.95\%. The multilinear basis's degree hierarchy becomes a liability without the snap constraint, since the $c_{ab}$ component's higher polynomial degree causes faster gradient decay through many layers of sigmoid activations. \textbf{(3)}~At $L{=}1$ (no depth interaction), all four cells cluster within 0.6pp, confirming that the basis effect is depth-amplified.

\subsection{Connection to theory}
\label{app:g-theory}

The 15pp STE-specific gap is consistent with:
\begin{itemize}
\item Lemma~\ref{lem:coverage} (Appendix~\ref{app:coverage}):
polynomial basis has $\E\|\bpsi\|_0{=}2.25$ vs corner $\|\boldsymbol{\phi}\|_0{=}1.0$ per sample, a $2.25\times$ more components updated on average, with a degree-ordered curriculum.
\item Proposition~\ref{prop:ste-noncancel} (Appendix~\ref{app:ste-covjac}):
single-polynomial Multilinear-STE backward has $O(1)$ per-layer signal, whereas IWP-STE with 25\% corner activation has uniform coverage but no curriculum.
\item Proposition~\ref{prop:basis-futility} (Appendix~\ref{app:basis-futility}):
no affine product reparameterization of the basis can fix the corner case: the multilinear $(1,a,b,ab)$ is a corner case of Proposition~\ref{prop:basis-futility} Part~(i) with $(p,q){=}(0,1)$ and is the only ``good'' STE configuration up to rescaling and coordinate permutation.
\end{itemize}

\subsection{Caveats}
\label{app:g-caveats}

\begin{itemize}
\item IWP-STE is our construction for this ablation;
\citet{ruttgers2026narrowing} do not propose or report it. The 17pp gap is \emph{not} a claim of improvement over a published method; it is an ablation demonstrating that basis matters under STE.
\item At $k{=}8\text{k}$ the basis effect was $\sim 22$pp; at
$k{=}64\text{k}$ it is 15.04pp.
\end{itemize}

\end{document}